\documentclass[a4paper]{article}

\usepackage[T1]{fontenc}
\usepackage[utf8]{inputenc}
\usepackage{adjustbox}
\usepackage{algpseudocode}
\usepackage{algorithm}
\usepackage{array}
\usepackage{amsthm}
\usepackage{amsmath}
\usepackage{amssymb}
\usepackage{color}
\usepackage{comment}
\usepackage{chngcntr}
\usepackage{dsfont} 
\usepackage[english]{babel}
\usepackage{float}
\usepackage{graphicx}
\usepackage{hhline}
\usepackage{hyperref}
\usepackage{multirow}
\usepackage{url}
\usepackage[caption=false,font=footnotesize]{subfig}
\usepackage[para,online,flushleft]{threeparttable}

\counterwithin{paragraph}{section}
\counterwithin{paragraph}{subsection}

\renewcommand{\b}[1]{{\boldsymbol{#1}}}   
\renewcommand{\L}{\La}
\renewcommand{\bf}[1]{\mathbf{#1}}
\newcommand{\G}{\mathcal{G}} 
\newcommand{\E}{\mathcal{E}} 
\newcommand{\F}{\mathcal{F}} 
\newcommand{\C}{\mathcal{C}}
\newcommand{\La}{\b{L}}
\newcommand{\bLambda}{\b{\Lambda}}
\newcommand{\V}{\mathcal{V}} 

\newcommand{\W}{\b{W}} 
\newcommand{\D}{\b{D}}
\newcommand{\T}{\mathcal{T}}

\newcommand{\Landau}{\mathcal{O}}
\newcommand{\U}{\b{U}}
\newcommand{\Uk}{\U_k}

\newcommand{\x}{\b{x}}
\newcommand{\X}{\b{X}}
\newcommand{\balpha}{\b{\alpha}}
\renewcommand{\P}{\b{P}}
\newcommand{\M}{\b{M}}
\newcommand{\lmax}{\lambda_{\rm max}}
\newcommand{\smax}{\sigma_{\rm max}}

\newcommand{\lel}{\lambda_\ell}
\newcommand{\blambda}{\b{\lambda}}

\newcommand{\Rbb}{\mathbb{R}} 

\newcommand{\Prob}[1]{\mathbb{P}\left[ #1 \right]}
\newcommand{\Esp}[1]{\mathbb{E}\left[ #1 \right]} 
\newcommand{\norm}[1]{\left\| #1 \right\|} 

\newcommand{\scp}[2]{\langle #1, #2 \rangle}

\DeclareMathOperator*{\argmin}{arg\,min}

\DeclareMathOperator*{\lkd}{LKD}
\DeclareMathOperator*{\kdd}{KDD}
\DeclareMathOperator*{\acc}{ACC}
\DeclareMathOperator*{\aci}{ACI}

\newtheorem{theorem}{Theorem} 
 \newtheorem{lemma}{Lemma}

\newcommand{\rkhs}{\ensuremath{\mathcal{H}_{\bf G}}}
\providecommand{\keywords}[1]{\textbf{\textit{Index terms---}} #1}
\begin{document}

\title{Compressive Embedding and Visualization using Graphs}

\author{Johan Paratte, Nathanaël Perraudin, Pierre Vandergheynst 
\thanks{
EPFL, Ecole Polytechnique Fédérale de Lausanne,
LTS2 Laboratoire de traitement du signal, CH-1015 Lausanne, Switzerland}
}



\maketitle

\begin{abstract} 
Visualizing high-dimensional data has been a focus in data analysis communities for decades, which has led to the design of many algorithms, some of which are now considered references (such as t-SNE for example). In our era of overwhelming data volumes, the scalability of such methods have become more and more important. In this work, we present a method which allows to apply any visualization or embedding algorithm on very large datasets by considering only a fraction of the data as input and then extending the information to all data points using a graph encoding its global similarity. We show that in most cases, using only $\Landau(\log(N))$ samples is sufficient to diffuse the information to all $N$ data points. In addition, we propose quantitative methods to measure the quality of embeddings and demonstrate the validity of our technique on both synthetic and real-world datasets. 
\end{abstract}




\section{Introduction}
DATA
visualization is usually equivalent to mapping high-dimensional features in low dimension using distance preserving dimensionality reduction. This process, finding a low-dimensional embedding of high-dimensional data, has drawn a lot of attention from researchers in different fields. 

Some methods are very fundamental such as Principle Componant Analysis (PCA) or Linear Discriminant Analysis (LDA). Other well known methods use the hypothesis that the data can be well approximated by a low-dimensional manifold, such as Laplacian Eigenmaps \cite{belkin2003laplacian}, Isomap \cite{tenenbaum2000global} or Local Linear Embedding (LLE) \cite{roweis2000nonlinear}. Another approach is to use a probabilistic model of both the high-dimensional and low-dimensional data distribution and optimize the distance preservation using the joint model. Examples of this approach are Stochastic Neighbor Embedding (SNE) \cite{hinton2002stochastic} and its popular extention t-SNE \cite{maaten2008visualizing}  or LargeVis \cite{tang2016visualizing}. We refer the interested reader to this work \cite{van2009dimensionality}, offering a comparative of numerous dimensionality reduction techniques.

From all those methods, two main pitfalls are the most prevalent. The first one is the lack of robustness to noisy real-world data and the second is bad scalability leading to unmanageable computing time for large datasets.

The first problem often arises when applying a global scheme which will work well on toy examples and fail on complex data, as the expected global model is only partially valid. A simple example would be the different results of Laplacian Eigenmaps which will yield the recovery of a perfect embedding for the Swissroll point cloud and poor results on large-scale complex and noisy data. This problem is traditionally mitigated by considering hypotheses on data to hold only locally, leading to techniques such as LLE, SNE and others. 

The second, more important, issue of scalability is essential in todays world of ubiquitous and overwhelming data. It is even more crucial now that the increase in data creation cannot be well compensated by the physical limits unsettling Moore's law. Essentially, this fundamental issue of scalibility is related to the notion of similarity. Indeed, the essential question one must be able to answer to represent data in low dimension is one of similarity : which data points are close to each other. This issue can be said to be fundamental because it naturally implies that the minimal complexity can only be super-linear, since one pass over each datapoint cannot be sufficient to infer a similarity matrix with a quadratic number of entries. Some of the popular methods mentioned above do have an intrinsic quadratic regime and parallelized or approximated variants that scale better, but at a cost. An illustrative example is t-SNE which is $\Landau(N^2)$ in its original implementation and is mostly used with an approximated and accelerated version (Barnes-Hut t-SNE \cite{van2014accelerating}) in $\Landau(N\log(N))$. 

As we saw, the two issues mentioned above are related to the concepts of locality and similarity. Expressing both notions naturally leads to the concept of a similarity graph whose edges link the closest points, weighted by the distance between them. This general idea is actually one of the most used tool when computing embeddings, either explicitly in methods such as Laplacian Eigenmaps or LargeVis, or implicitly, using probability distributions as random walk matrices (e.g. SNE). Of course, constructing a similarity graph has the same complexity issue as the one mentioned above. This is why approximated sparse nearest-neighborgs graphs are often used in practice, as they can be computed very efficiently using Approximated Nearest Neighbors (ANN) techniques (e.g. FLANN \cite{flann_pami_2014}). 

In this work, we propose a general framework for accelerating any embedding algorithm using a graph encoding the data similarity. Our technique is supported by modern tools of Graph Signal Processing allowing to use the graph at both local and global scales. The main idea is to use only a subset of the data on which to apply an embedding algorithm and then diffuse the information using the graph. Our main contribution which we call Compressive Embedding (CE) is made possible by two complementary mechanisms : a graph sampling scheme to create the sketch and diffusion routines to extend the information on the sketch to all data points.  

\paragraph{Contributions}
Below we summarize the main contributions of this work :
\begin{itemize}
\item graph sampling schemes and theorems stating the minimum number of samples necessary to capture energy everywhere
\item transductive learning algorithms to extend the embedding information computed on the samples to all datapoints using localized low pass graph filters
\item new quantitative measures of the quality of the visualizations based on graph cuts and localized filters 
\item experiments on synthetic and real data sets showing the superior scalability of this method compared to the state-of-the-art
\end{itemize}

\paragraph{Organization}
The paper is organized as follows. In Section~\ref{sec:background}, we recall the fundamentals of graph signal processing and define the notations. Section~\ref{sec:randomsampling} develops the results on our sampling method based on the energy of localized kernels. Section~\ref{sec:metrics} uses localized filters to define generalized metrics used in the following sections. Section ~\ref{sec:transductive} describes the different methods to extend the information from the sampled nodes to all data points. Section~\ref{sec:embedquality} describes our proposed methods to compute a quantitative measure of the quality of embeddings. In Section~\ref{sec:experiments}, we show the validity and benefits of our method and compare with the state-of-the-art through several experiments. Finally, Section~\ref{sec:conclusion} proposes interesting open problems in the domain as well as potential future work to address.

\section{Background} \label{sec:background}

\paragraph{Graph nomenclature}
Let us define $\G = (\V, \E, \W)$ as an undirected weighted graph where $\V$ is the set of vertices and $\E$ the set of edges representing connections between nodes in $\V$. The vertices $v \in \V$ of the graph are ordered from $1$ to $N=|\V|$. The matrix $\W$, which is symmetric and positive, is called the weighted adjacency matrix of the graph $\G$. The weight $\W_{ij}$ represents the weight of the edge between vertices $v_i$ and $v_j$ and a value of 0 means that the two vertices are not connected. The degree $d(i)$ of a node $v_i$ is defined as the sum of the weights of all its edges $d(i)=\sum_{j=1}^N \W_{ij}$. Finally, a graph signal is defined as a vector of scalar values over the set of vertices $\V$ where the $i$-th component of the vector is the value of the signal at vertex $v_i$.

\paragraph{Spectral theory}
The combinatorial Laplacian operator $\L$ can be defined from the weighted adjacency matrix as $\La = \mathbf{D}-\W$ with $\mathbf{D}$ being the degree matrix defined as a diagonal matrix with $\D_{ii}=d(i)$. One alternative and often used Laplacian definition is the normalized Laplacian $\La_n = \mathbf{D}^{-\frac{1}{2}} \La \mathbf{D}^{-\frac{1}{2}} = \mathbf{I} - \mathbf{D}^{-\frac{1}{2}} \W \mathbf{D}^{\frac{1}{2}}$. Since the weight matrix $\W$ is symmetric positive semi-definite, so is $\L$ by construction. By application of the spectral theorem, we know that $\L$ can be decomposed into an orthonormal basis of eigenvectors noted $\{ \mathbf{u}_\ell \}_{\ell=0, 1,\ldots, N-1}$. The ordering of the eigenvectors is given by the eigenvalues noted $\{ \lambda_\ell \}_{\ell=0, 1,\dots, N-1}$ sorted in ascending order $0=\lambda_0 \leq \lambda_1 \leq \lambda_2 \leq \ldots \leq \lambda_{N-1} = \lambda_{\rm max}$. In a matrix form we can write this decomposition as $\L = \U\Lambda \U^*$ with $\U = (\mathbf{u}_1 | \mathbf{u}_2 | \ldots | \mathbf{u}_{N-1} )$ the matrix of eigenvectors and $\Lambda$ the diagonal matrix containing the eigenvalues in ascending order. Given a graph signal $\x$, its graph Fourier transform is thus defined as $\hat{\x} = \F (\x) = \U^*\x$, and the inverse transform $\x = \F^{-1}(\hat{\x})=\U\hat{\x}$. It is called a Fourier transform by analogy to the continuous Laplacian whose spectral components are Fourier modes, and the matrix $\U$ is sometimes referred to as the graph Fourier matrix (see e.g., \cite{chung1997spectral}). By the same analogy, the set $\{ \sqrt{\lambda_\ell} \}_{\ell=0, 1,\ldots, N-1}$ is often seen as the set of graph frequencies~\cite{shuman2013vertex}.

\paragraph{Graph filtering}
In traditional signal processing, filtering can be carried out by a pointwise multiplication in Fourier. Thus, since the graph Fourier transform is defined, it is natural to consider a filtering operation on the graph using a multiplication in the graph Fourier domain. To this end, we define a graph filter as a continuous fonction $g:\Rbb_+ \rightarrow \Rbb$ directly in the graph Fourier domain. If we consider the filtering of a signal $\x$, whose graph Fourier transform is written $\hat{\x}$, by a filter $g$ the operation in the spectral domain is a simple multiplication $\hat{\x'}[\ell] =  g(\lambda_\ell) \cdot \hat{\x}[\ell]$, with $\x'$ and $\hat{\x'}$ the filtered signal and its graph Fourier transform respectively. Using the graph Fourier matrix to recover the vertex-based signals we get the explicit matrix formulation for graph filtering:
$$\x' = \U g(\Lambda) \U^* \x ,$$
where $g(\Lambda) = \text{diag}(g(\lambda_0), g(\lambda_1), \ldots, g(\lambda_{N-1}))$. The graph filtering operator $g(\L) := \U g(\Lambda) \U^*$ is often used to reformulate the graph filtering equation as a simple vector-matrix operation $\x' = g(\La) \x$.

Since the filtering equation defined above involves the full set of eigenvectors $\U$, it implies the diagonalization of the Laplacian $\La$ which is costly for large graphs. To circumvent this problem, one can represent the filter $g$ as a polynomial approximation, since polynomial filtering only involves the multiplication of the signal by a power of $\La$ of the same order as the polynomial. Filtering using good polynomial approximations can be done using Chebyshev or Lanczos polynomials \cite{hammond2011wavelets, susnjara2015accelerated}.

\paragraph{Localization operator}

The concept of translation, which is well defined in traditional signal processing cannot be directly applied to graphs, as they can be irregular. However, inspired by the notion of translation, we can define the localization of a function $g$ defined on the graph spectrum as a convolution with a Kronecker delta $ \widehat{\T_ig[\ell]} = g(\lel) \cdot \hat{\delta_i} = g(\lel) \cdot \mathbf{u}_\ell[i] $, where $\T$ is called the localization operator, and $ \T_i $ means localization at vertex $i$. Going back to the vertex domain, we get :

\begin{equation*}
\label{def:localization_operator}
\T_i g [n] = \mathcal{F}^{-1} \left( g \cdot \hat{\delta_i} \right)[n] = \sum_{\ell=0}^{N-1} g(\lambda_\ell) \mathbf{u}^*_\ell[i] \mathbf{u}_\ell[n] = \left(g(\La)\right)_{in}.
\end{equation*}

The reason for calling $\T_i$ a localization operator comes from the fact that for smooth functions $g$, $\T_ig$ is localized around the vertex $i$. The proof of this result and more information on the localization operator can be found in \cite{shuman2016vertex}. The localization of filters is quite naturally called atoms as a filtering operation of a signal $\x$ using a filter $g$ can be expressed as $\x'[i] = \scp{\x}{\T_ig}$. 

\paragraph{Additional notation}
We use $\norm{\b{A}}_{op}=\sup_{\x\neq0}\frac{\norm{\b{A}\x}_2}{\norm{\x}_2}$ for the induced norm of the matrix $\b{A}$ and $\norm{\b{A}}_F=\sqrt{\sum_i\sum_j \b{A}_{ij}}$ for the Froebenius norm. The maximum eigenvalue of a matrix is written $\smax(\b{A})$. 

We reserve the number notation for vectors. For example, we write the $\ell_2$ Euclidean norm as $\norm{\x}_2=\sqrt{\sum_i\x_i}$ and the $\ell_\infty$ uniform (sup) norm $\norm{\x}_\infty=\max_i|\x_i|$. We abusively use the $\ell_0$ to count the number of non-zero elements in a vector. Furthermore, when an univariate function $g$ is applied to a vector $\blambda$, we mean $[g(\blambda)]_i=g(\blambda_i)$. As a result, $\norm{g(\blambda)}_0=k$ is the number of eigenvalues where $g(\lambda_\ell) \neq 0$.

Given a kernel $g$, we define $\Uk$ as a $N\times k$ matrix made of the $k$ columns of $\U$ where $g(\lambda_\ell) \neq 0$. Similarly, we denote $\bLambda_k$ the $k\times k$ diagonal matrix containing the associated eigenvalues. Note that we have
\begin{equation*}
g(\L) = \U g(\bLambda) \U^* = \U_k g(\bLambda_k) \U_k^* = \Uk\Uk^*g(\L).
\end{equation*} 


\section{Random sampling on graphs}
\label{sec:randomsampling}

In this section, we first define a graph sampling schemes and then prove related theoretical limits. In particular, it is of particular interest to understand the number of samples needed in order to diffuse energy on every node by localizing filters on the samples. We will prove that the number of samples needed is direclty linked with the rank of the filter. 

\subsection{Adaptive sampling scheme} \label{sec:adapted_sampling_scheme}

Let us define the probability distribution $\mathcal{P}$ represented by
a vector $\b{p}\in\Rbb^{N}$. We use two different sampling schemes.
Uniform sampling is given by the probability vector
\begin{equation*}
\b{p}_{i}=\frac{1}{N},
\end{equation*}
and adapted sampling is given by 
\begin{equation*}
\b{p}_{i}=\frac{\norm{\T_{i}g}_{2}^{2}}{\norm{g(\b{\lambda})}_{2}^{2}}.
\end{equation*}
Remember that we have $\sum_{i}\|\T_{i}g\|_{2}^{2}=\|g\|_{2}^{2}$
, implying that $\sum_{i}p_{i}=1$. Let us associate the matrix
\begin{equation*}
P:=\mbox{diag}(p)\in\Rbb^{N\times N}
\end{equation*}
 to $p.$

Then, we draw independently (with replacement) $M$ indices $\Omega:=\{\omega_{1},\dots,\omega_{M}\}$
from the set $\{1,\dots,N\}$ according to the probability distribution
$\b{p}$. We have
\begin{equation*}
\Prob{\omega_{j}=i}=\b{p}_{i},\hspace{1em}\forall i\in\{1,\dots,N\},\hspace{1em}\forall j\in\{1,\dots,M\}.
\end{equation*}
For any signal $\b{x}\in\Rbb^{N},$ defined on the vertices of the
graph, its sampled version $\b{y}\in\Rbb^{M}$ satisfies
\begin{equation*}
\b{y}_{j}:=\b{x}_{\omega_{j}}\quad\forall j\in\{1,\dots,M\}.
\end{equation*}

Finally, the downsampling matrix $M\in\Rbb^{M\times N}$ is defined as
\begin{equation*}
\M_{ij}=\begin{cases}
1 & \mbox{if }i=\omega_{j}\\
0 & \text{otherwise,}
\end{cases}
\end{equation*}
for all $i\in\{1,\dots,N\}$ and $j\in\{1\text{,}\dots,M\}.$ Note
that $\b{y}=\M\b{x}$. 

\subsection{Embedding Theorems}

The first theorem shows that given enough samples, the random projection $\M\P^{-\frac{1}{2}}g(\L)\x$ conserves the energy contained in $g(\L)\x$. In this sense, given enough samples, it is an embedding of $g(\L)\x$.
\begin{theorem} \label{theo:sampling_gL}
Given a graph $\G$ and a kernel $g$ with a given rank $\norm{g(\blambda)}_{0}=k$, given $\delta>0$ and using the sampling scheme of Section~\ref{sec:adapted_sampling_scheme}, if 
\[
M\geq2\frac{1}{\delta^{2}}\frac{\norm{g(\blambda)}_{2}^{2}}{\norm{g(\blambda)}_{\infty}^{2}}\left(1+\frac{\delta}{3}\right)\log\left(\frac{2k}{\epsilon}\right)
\]
we have with a probability of $1-\epsilon$ for all $\x$: 
\begin{equation}
\left|\frac{\frac{1}{M}\norm{\M\P^{-\frac{1}{2}}g(\L)\x}_{2}^{2}-\norm{g(\L)\x}_{2}^{2}}{\norm{g(\blambda)}_{\infty}^{2}}\right|\leq\delta\|\U_{k}^{*}\x\|_{2}^{2}\leq\delta\|\x\|_{2}^{2}.
\end{equation}
\end{theorem}
Note that the above expression is normalized by $\norm{g(\blambda)}_{\infty}^{2}$ in order to remove the scaling factor of the kernel $g$.

Let us now analyze the most important term of the bound: 
\begin{equation}
\frac{\norm{g(\blambda)}_{2}^{2}}{\norm{g(\blambda)}_{\infty}^{2}} = \frac{\sum_\ell g^2(\lel)}{\max_\ell g^2(\lel)}.
\end{equation}
It is a measure of concentration of the kernel on its support. It is maximized with the value $\frac{\norm{g(\blambda)}_{2}^{2}}{\norm{g(\blambda)}_{\infty}^{2}}=k$ when $g$ is a rectangle.
In general, it will be small for concentrated kernels. For example, a rapidly decreasing kernel such as the heat kernel ($g(x)=e^{-x\tau}$) will lead to a very small ratio. 

Note that contrarily to almost all bound available in the literature this bound does not require the kernel to be low rank but only concentrated.
For a comparison~\cite[Corollary 2.3]{puy2016random} requires
\begin{equation*}
M \geq \frac{3}{\delta^2} k \log\left(\frac{2k}{\epsilon}\right).
\end{equation*}

\paragraph{Optimality of the sampling scheme.} Although we have no formal proof of optimality, the sampling scheme presented in Section~\ref{sec:adapted_sampling_scheme} is a good candidate. Indeed, when reading the proof of Theorem~\ref{theo:sampling_gL}, the reader may notice that it minimizes the number of samples $M$.

Building on top of Theorem~\ref{theo:sampling_gL}, we establish a lower bound on the number of samples required by Algorithm~\ref{algo:ce} to capture enough information from each node with a given confidence level. It will ensure that the information diffused from the samples can reach all nodes.

\begin{theorem} \label{theo:sampling_tig_exact}
Using the sampling scheme described in Section~\ref{sec:adapted_sampling_scheme}, for $\delta>0$, a graph $\G$ and a kernel $g$ such that $\norm{g(\blambda)}_{0}=k$, each node $i$ is guaranteed with a probability
$1-\epsilon$ to have
\[
\frac{\frac{1}{M}\norm{\M\P^{-\frac{1}{2}}\T_{i}g}_{2}^{2}}{\norm{\T_{i}g}_{2}^{2}}\geq1-\delta,
\]
given that the number of samples satisfies
\[
M\geq\frac{2a}{\delta^{2}}\left(1+\frac{\delta}{3}\right)\log\left(\frac{k}{\epsilon}\right),
\]
where $a = \frac{\norm{g(\blambda)}_{2}^{2}\norm{g(\blambda)}_{\infty}^{2}\norm{\Uk^{*}\b{\delta}_{i}}_{2}^{4}}{\norm{\T_{i}g}_{2}^{4}}$.
\end{theorem}
Theorem~\ref{theo:sampling_tig_exact} warrants that given enough samples $M$, Algorithm~\ref{algo:ce} captures with some probability $1-\epsilon$ (close to $1$), at least a good percentage of the energy at node $i$. The factor $a$ is always greater than $1$ and varies depending on the shape of the kernel $g$ and of the graph eigenvectors. However it is $\Landau(k)$ and exactly equal to $k$ if $g$ is a rectangular kernel. Indeed, a simple transformation shows that
\begin{equation*}
a = \frac{\norm{g(\blambda)}_{2}^{2}\norm{g(\blambda)}_{\infty}^{2}\norm{\Uk^{*}\b{\delta}_{i}}_{2}^{4}}{\norm{\T_{i}g}_{2}^{4}}
=
\frac{\sum_\ell g^2(\lel)}{\max_{\ell}\left|g^{2}(\lambda_{\ell})\right|}\left( \frac{\max_{\ell}\left|g^{2}(\lambda_{\ell})\right|\sum_{\ell\in\mathcal{K}}\b{u}_{\ell}^{2}[i]}{\sum_{\ell}g^{2}(\lambda_{\ell})\b{u}_{\ell}^{2}[i]} \right)^2. \end{equation*}
The first term is smaller than $k$ but is usually close to $k$ for a kernel close to a rectangle. The second term is greater than $1$ but close to $1$ given that the kernel is close to a rectangle. 

Problematically, this bound becomes loose if the kernel $g$ has a large rank because of the term $\sum_{\ell\in\mathcal{K}}\b{u}_{\ell}^{2}[i]$. To cope with this problem we can use another kernel $g^\prime$ that is a low-rank approximation of $g$.
\begin{theorem} \label{theo:sampling_tig_approx}
Given a graph $\G$, let $g^{\prime}$ (with $\norm{g^\prime(\blambda)}_{0}=k$) to be the rank $k$ approximation of the kernel $g$,
i.e.,
\[
g(\lel)=\begin{cases}
g^{\prime}(\lel) & \text{for the the \ensuremath{k} greatest values of \ensuremath{|g(\lel)|}}\\
0 & \text{otherwise.}
\end{cases}
\]
Using the sampling scheme described in Section~\ref{sec:adapted_sampling_scheme} with the kernel $g$, for $\delta>0$, each node $i$ is assured with a probability $1-\epsilon$ to have
\[
\frac{\frac{1}{M}\norm{\M\P^{\frac{1}{2}}\T_{i}g}_{2}^{2}}{\norm{\T_{i}g}_{2}^{2}}\geq1-\delta-\frac{\norm{\T_{i}\left(|g^{\prime}|-|g|\right)}_{2}^{2}}{\norm{\T_{i}g}_{2}^{2}}
\]
providing the number of samples satisfies\footnote{Note that $\norm{\T_{i}g}_{2}^{2}\geq\norm{\T_{i}g^\prime}_{2}^{2}$.}
\[
M\geq2\frac{1}{\delta^{2}}\frac{\norm{g^\prime(\blambda)}_{2}^{2}\norm{g^\prime(\blambda)}_{\infty}^{2}\norm{\Uk^{*}\b{\delta}_{i}}_{2}^{4}}{\norm{\T_{i}g}_{2}^{4}}\left(1+\frac{\delta}{3}\right)\log\left(\frac{k}{\epsilon}\right).
\]
\end{theorem}
Using Theorem~\ref{theo:sampling_tig_approx}, the number of samples $M$ required can be highly reduced. Indeed, when the kernel $g$ is well concentrated but not low rank, we trade some approximation error encoded by $\frac{\norm{\T_{i}\left(|g^{\prime}|-|g|\right)}_{2}^{2}}{\norm{\T_{i}g}_{2}^{2}}$ (which will be low if $g$ is concentrated) but we will need a smaller number of samples due to the fact that $g^\prime$ is low rank. This theorem can be interesting for a heat kernel for example. 



\section{Metrics based on localized filters}
\label{sec:metrics}

Before moving on to the information diffusion from the samples, we need to take a closer look to localized filters and in particular see how they can be used to measure distances or correlations between nodes.  

\subsection{Localized Kernel Distance}

Since localized filters are proven to be concentrated in the vertex domain (see \cite[Theorem 1]{shuman2013vertex}), it seems natural to use them to get geodesic measures or correlations between nodes. To this end, we introduce the Localized Kernel Distance (LKD), which is defined as :

\begin{equation} \label{eq:lkd_definition}
\lkd(i, j) = 1 - \frac{\T_ig^2[j]}{\|\T_ig\| \| \T_jg\|}.
\end{equation}

Let us now examine its properties by stating the following theorem:

\begin{theorem} \label{theo:lkd_pseudosemimetric}
The space $(\V, \lkd)$ with $\V$ the vertex set of a graph and $\lkd$ as defined in \ref{eq:lkd_definition} is a pseudosemimetric space, that is, for every $x, y \in \V$:
\begin{enumerate}
\item $\lkd(x, y) \geq 0$
\item $\lkd(x, x) = 0  $
\item $\lkd(x, y) = \lkd(y, x)$
\end{enumerate} 

\end{theorem}

\begin{proof}
First, let us derive an alternative form of \eqref{eq:lkd_definition} :

\begin{equation}\label{lkd_definition_2}
\lkd(x, y) = 1 - \frac{\langle \T_xg , \T_yg \rangle}{\|\T_xg\| \| \T_yg\|}
\end{equation}

This can be derived as follows :

\begin{eqnarray*}
\lkd(x, y) & = & 1 - \frac{\T_xg^2[y]}{\|\T_xg\| \| \T_yg\|} \\
& = & 1 - \frac{\sum_\ell g(\b \lambda_\ell)^2 \b u_{\ell}^*[x] \b u_{\ell}[y]}{\|\T_xg\| \| \T_yg\|} \\
& = & 1 - \frac{\sum_\ell (g(\b \lambda_\ell) \b u_{\ell}^*[x]) (g(\b \lambda_\ell) \b u_{\ell}[y])}{\|\T_xg\| \| \T_yg\|} \\
& = & 1 - \frac{\sum_\ell (g(\b \lambda_\ell) \b u_{\ell}^*[x]) (g(\b \lambda_\ell) \b u_{\ell}^*[y])\sum_n \b u_{\ell}[n]^2}{\|\T_xg\| \| \T_yg\|} \\
& = & 1 - \frac{\sum_n \sum_\ell (g(\b \lambda_\ell) \b u_{\ell}^*[x]\b u_\ell[n]) (g(\b \lambda_\ell) \b u_{\ell}^*[y] \b u_{\ell}[n])}{\|\T_xg\| \| \T_yg\|} \\
& = & 1 - \frac{\langle \T_xg , \T_yg \rangle}{\|\T_xg\| \| \T_yg\|}
\end{eqnarray*}

Now let us verify the properties one by one :
\begin{enumerate}
\item We have using \eqref{lkd_definition_2} :
\begin{eqnarray*}
\lkd(x, y) & = &  1 - \frac{\langle \T_xg , \T_yg \rangle}{\|\T_xg\| \| \T_yg\|} \\
& \geq &  0
\end{eqnarray*}
where the last inequality stands because $\langle \T_xg , \T_yg \rangle \leq \|\T_xg\| \| \T_yg\|$ (Cauchy-Schwartz inequality). 

\item Let us verify that $x = y \Rightarrow \lkd(x, y) = 0$ :
\begin{eqnarray*}
\lkd(x, y) & = & \lkd(x, x) \\
& = & 1 - \frac{\T_xg^2[x]}{\|\T_xg\| \| \T_xg\|} \\
& = &  1 - \frac{\sum_\ell g(\b \lambda_\ell)^2 \b u_{\ell}^*[x] \b u_{\ell}[x]}{\|\T_xg\|^2} \\
& = &  1 - \frac{\sum_\ell (g(\b \lambda_\ell) \b u_{\ell}[x])^2 \sum_n \b u_{\ell}[n]^2}{\|\T_xg\|^2} \\
& = &  1 - \frac{\sum_n \sum_\ell (g(\b \lambda_\ell) \b u_{\ell}[x] \b u_{\ell}[n])^2}{\|\T_xg\|^2} \\
& = &  1 - \frac{\|\T_xg\|^2}{\|\T_xg\|^2} \\
& = & 0
\end{eqnarray*}

\item Finally, we have 
\begin{eqnarray*}
\lkd(x, y) & = & 1 - \frac{\T_xg^2[y]}{\|\T_xg\| \| \T_yg\|} \\
& = &  1 - \frac{\sum_\ell g(\b \lambda_\ell)^2 \b u_{\ell}^*[x] \b u_{\ell}[y]}{\|\T_xg\| \| \T_yg\|} \\
& = & 1 - \frac{\sum_\ell g(\b \lambda_\ell)^2 \b u_{\ell}^*[y] \b u_{\ell}[x]}{\|\T_xg\| \| \T_yg\|} \\
& = & 1 - \frac{\T_yg^2[x]}{\|\T_xg\| \| \T_yg\|} \\
& = & \lkd(y, x)
\end{eqnarray*}

\end{enumerate}

\end{proof}

\begin{theorem} \label{theo:lkd_semimetric}
The space $(\V, \lkd)$ with $\V$ the vertex set of a graph and $\lkd$ as defined in \ref{eq:lkd_definition}, with $g$ constant, is a semimetric space, that is, for every $x, y \in \V$:
\begin{enumerate}
\item $\lkd(x, y) \geq 0$
\item $\lkd(x, y) = 0 \Leftrightarrow x = y $
\item $\lkd(x, y) = \lkd(y, x)$
\end{enumerate} 

\end{theorem}

\begin{proof}

Properties 1 and 3, as well as the backward implication are still valid as stated in Theorem~\ref{theo:lkd_pseudosemimetric}.

Now let us check that $\lkd(x, y) = 0 \Rightarrow x = y$.

We want to do it by contradiction and thus search any $x,y$, $x \neq y$ for which $\lkd(x, y) = 0$, implying :

\begin{equation}
\langle \T_xg , \T_yg \rangle  =  \|\T_xg\| \| \T_yg\|
\end{equation}


We can rewrite this equality as :

\begin{equation}
\sum_\ell g(\b \lambda_\ell)^2 \b u_{\ell}^*[x] \b u_{\ell}[y] = \sqrt{\sum_\ell g(\b \lambda_\ell)^2 \b u_{\ell}^2[x]} \sqrt{\sum_\ell g(\b \lambda_\ell)^2 \b u_{\ell}^2[y]}
\end{equation}

For $g(x) = c$, with $c > 0$ a constant, the left hand side is :
\begin{equation}
\sum_\ell g(\b \lambda_\ell)^2 \b u_{\ell}^*[x] \b u_{\ell}[y] = c^2 \sum_\ell \b u_{\ell}^*[x] \b u_{\ell}[y] = 0
\end{equation}

The last equality comes from the fact that two lines of an orthonormal matrix are orthogonal, and $x \neq y$. 

Now the right-hand side is :

\begin{equation}
\sqrt{\sum_\ell g(\b \lambda_\ell)^2 \b u_{\ell}^2[x]} \sqrt{\sum_\ell g(\b \lambda_\ell)^2 \b u_{\ell}^2[y]} = c^2 \sum_\ell \b u_{\ell}^2[x] \sum_\ell \b u_{\ell}^2[y] = c^2
\end{equation}

with the last equality coming from the fact that $\U$ is an orthonormal basis.

Now, since $0 \neq c^2$ we have a contradiction, and thus the proof is completed. 


\end{proof}

\subsection{Kernelized Diffusion Distance}

Another approach to use localized atoms to define distances is to measure the norm of the difference between a filter localized at two different nodes. We call it the Kernelized Diffusion Distance and define it as:

\begin{equation} \label{eq:kdd_definition}
\kdd(i, j) = \| \T_ig - \T_jg \|, 
\end{equation}

where $g$ is a kernel defined in the graph spectral domain. Before going further, and as it will be useful later, let us derive a corollary definition of \ref{eq:kdd_definition} :

\begin{equation} \label{eq:kdd_definition_spectral}
\kdd(i, j) = \sqrt{\sum_{\ell} g(\b\lambda_{\ell})^2 (\b u^*_{\ell}[i] - \b u^*_{\ell}[j])^2 }. 
\end{equation}

This alternative definition can be quickly derived as follows :

\begin{eqnarray*}
\kdd(i, j)^2 &=& \| \T_ig - \T_jg \| \\
& = & \sum_n \left( \sum_{\ell} g(\b\lambda_{\ell}) \b u^*_{\ell}[i] \b u_{\ell}[n] - \sum_{\ell} g(\b\lambda_{\ell}) \b u^*_{\ell}[j]) \b u_{\ell}[n] \right)^2 \\
& = & \sum_n \left( \sum_{\ell} g(\b\lambda_{\ell}) (\b u^*_{\ell}[i] - \b u^*_{\ell}[j]) \b u_{\ell}[n] \right)^2 \\
& = & \sum_n \sum_{\ell} g(\b\lambda_{\ell})^2 (\b u^*_{\ell}[i] - \b u^*_{\ell}[j])^2 \b u^2_{\ell}[n] \\
& = & \sum_{\ell} g(\b\lambda_{\ell})^2 (\b u^*_{\ell}[i] - \b u^*_{\ell}[j])^2 \sum_n \b u^2_{\ell}[n] \\
& = & \sum_{\ell} g(\b\lambda_{\ell})^2 (\b u^*_{\ell}[i] - \b u^*_{\ell}[j])^2
\end{eqnarray*}

which implies \ref{eq:kdd_definition_spectral} by taking the square root on both sides. 

Let us now examine the properties of the KDD by stating the following theorem:

\begin{theorem} \label{theo:kdd_pseudometric}
The space $(\V, \kdd)$ with $\V$ the vertex set of a graph and $\kdd$ as defined in \ref{eq:kdd_definition} is a pseudometric space, that is, for every $x, y, z \in \V$:
\begin{enumerate}
\item $\kdd(x, y) \geq 0$
\item $\kdd(x, y) = \kdd(y, x)$
\item $\kdd(x, z) \leq \kdd(x, y) + \kdd(y, z)$
\end{enumerate} 

\end{theorem}

\begin{proof}
Let us verify the properties in order :

\begin{enumerate}
\item This property holds trivially due to the positivity of the norm $\| . \|$. 
\item We have 
\begin{eqnarray*}
\kdd(x, y) & = & \| \T_xg - \T_yg \| \\
& = &  \sqrt{\sum_{\ell} g(\b\lambda_{\ell})^2 (\b u^*_{\ell}[x] - \b u^*_{\ell}[y])^2 } \\
& = &  \sqrt{\sum_{\ell} g(\b\lambda_{\ell})^2 (\b u^*_{\ell}[y] - \b u^*_{\ell}[x])^2 } \\
& = & \| \T_yg - \T_xg \| \\
& = & \kdd(y, x)
\end{eqnarray*}

\item We have
\begin{eqnarray*}
\kdd(x, z) &=& \| \T_xg - \T_zg \| \\
&=& \| \T_xg - \T_yg + \T_yg - \T_zg \| \\
&\leq& \| \T_xg - \T_yg \| + \| \T_yg - \T_zg \| \\
&=& \kdd(x, y) + \kdd(y, z)
\end{eqnarray*}
which holds using the triangle inequality for vectors. 

\end{enumerate}

\end{proof}

Now that we proved that the KDD is a pseudo-metric, we only need to have the identity of the indiscernibles, i.e. $\kdd(i, j) = 0 \Leftrightarrow i = j$ to prove it is a metric. However, we can only do it using an additional hypothesis on $g$. This is formulated in the following theorem~:

\begin{theorem} \label{theo:kdd_metric}
The space $(\V, \kdd)$ with $\V$ the vertex set of a graph and $\kdd$ as defined in \ref{eq:kdd_definition}, with $g$ being full rank, is a metric space, that is, for every $x, y, z \in \V$:
\begin{enumerate}
\item $\kdd(x, y) \geq 0$
\item $\kdd(x, y) = \kdd(y, x)$
\item $\kdd(x, z) \leq \kdd(x, y) + \kdd(y, z)$
\item $\kdd(x, y) = 0 \Leftrightarrow x = y$
\end{enumerate} 

\end{theorem}

\begin{proof}

Properties 1-3 are still valid as stated in Theorem~\ref{theo:kdd_pseudometric}.

Now let us check Property 4. 
\begin{itemize}

\item We first prove $x = y \Rightarrow \kdd(x, y) = 0$ :
\begin{eqnarray*}
d_g(x, y) & = & d_g(x, x) \\
& = & \| \T_xg - \T_xg \| \\
& = &  \sqrt{\sum_{\ell} g(\b\lambda_{\ell})^2 (\b u^*_{\ell}[x] - \b u^*_{\ell}[x])^2 } \\
& = & 0
\end{eqnarray*}

\item Now let us check that $\kdd(x, y) = 0 \Rightarrow x = y$.
We do it by contradiction and thus want to find any pair $x,y$, $x \neq y$ for which $\kdd(x, y) = 0$. 

In particular we need that :

\begin{equation} \label{eq:kdd_contradiction}
\kdd(x, y) = \sqrt{\sum_{\ell} g(\b\lambda_{\ell})^2 (\b u^*_{\ell}[x] - \b u^*_{\ell}[y])^2 } = 0
\end{equation}

with $x \neq y$. Since $g$ is full rank then $g(\b \lambda_{\ell}) > 0$, $\forall \ell$ and thus the only way for \eqref{eq:kdd_contradiction} to hold is if $\b u^*_{\ell}[x] = \b u^*_{\ell}[y]$, $\forall \ell$. In other words it would imply that the lines $x$ and $y$ of $\U$ are identical. Since $\U$ is a basis, it implies that all its lines are orthonormal, which means there exist no pair $x, y$ such as \eqref{eq:kdd_contradiction} hold, and thus the contradiction is established, which concludes the proof.  

\end{itemize}

\end{proof}

\paragraph{Diffusion distance}

As was hinted in the name, the distance defined in \eqref{eq:kdd_definition} happens to be a generalized diffusion distance. Indeed, taking its spectral formulation we have : 

\begin{equation}\label{eq:kdd_diffusion_distance}
d_g(i, j) = \sqrt{\sum_{\ell} g(\b\lambda_{\ell})^2 (\b u^*_{\ell}[i] - \b u^*_{\ell}[j])^2 } = D_t(i, j),
\end{equation}

where $D_t(i, j)$ is the diffusion distance associated to specific kernels depending on $t$ (i.e. the diffusion parameter). If we take two common definitions of the diffusion distance, the original works of \cite{nadler2005diffusion} and \cite{coifman2006diffusion} use a kernel of the form $g(x) = x^t$ and the Graph Diffusion Distance defined in \cite{hammond2013graph} uses the heat kernel $g(x) = e^{-tx}$. 

%
%
%


\section{Graph transductive learning}
\label{sec:transductive}

In this section we want to cast the problem of diffusing the information obtained on a few samples of the data (e.g. using sampling schemes such as defined in Section~\ref{sec:adapted_sampling_scheme}) in a transductive inference framework. In this setting, we are observing a label field or signal $\b x$ only at a subset of vertices $S \subset V$, i.e $\b y_i = \b x[i]$, $\forall i \in S$, with $\b y $ being the observed signal also called the label function. The goal of transductive learning is to predict the missing signal/labels using both the observed signal and the remaining data points. 

\subsection{Global graph diffusion}

Solutions of transductive inference using graphs can be solved in a number of ways, for example using Tikhonov regression~:

\begin{equation}\label{eq:tikhonov_regression}
\argmin_{\b x} \|\b y - \M \b x \|_2^2 + \mu \b x^t \L \b x,
\end{equation}

where $\M$ is the sampling operator and $\L$ the graph Laplacian. An alternative to the use of the Dirichlet smoothness constraint is to use graph Total Variation (TV). The regression would thus become :

\begin{equation}\label{eq:tv_regression}
\argmin_{\b x} \| \b y - \M \b x \|_2^2 + \mu \| \nabla_{\G} \b x \|_1
\end{equation}

with $ \nabla_{\G} \b x = \left( \sqrt{\W_{i,j}} (\b x[i] - \b x [j]) \right)$, $\forall (v_i, v_j) \in \E$. 

For large scale learning, solving the optimization problems as described above can be too expensive and one typically uses accelerated descent methods.

\subsection{RKHS transductive learning on graphs}

\subsubsection{Motivation}

Our first contribution is to replace the smoothness term arising in \ref{eq:tikhonov_regression} by constraining the solution to belong to the finite dimensional Reproducing Kernel Hilbert Space (RKHS) $\rkhs$ corresponding to the graph kernel $\bf G = g(\L)$, for some filter $g$. In this case, we instead solve the following problem~:

\begin{equation*}
\textrm{arg}\min_{\b x \in \rkhs} \|\b y - \M \b x \|_2^2
\end{equation*}

and show that the solution is given by a simple low-pass filtering step applied to the labelled examples. 

%

\subsubsection{Transductive learning and graph filters}
In this section, we formulate transductive learning as a finite dimensional regression problem. This problem is solved by constructing a reproducing kernel Hilbert space from a graph filter, which controls the smoothness of the solution and provides a fast algorithm to compute it. 

\paragraph{An empirical reproducing kernel Hilbert space}
Let $g$ be a smooth, strictly positive function defining a graph filter as defined in Section~\ref{sec:background}. The graph filter defines the following matrix~:
\begin{equation*}
\mathbf{G}[i,j] = g(\L)[i,j] = \T_ig [j],
\end{equation*}

where $\T_i$ is the localisation operator at vertex $i$. Since the filter is strictly positive definite,  $\mathbf{G}$ is positive definite and can be written as the Gram matrix of a set of linearly independent vectors. To see this, we use the spectral representation~:
\begin{eqnarray*}
	\bf G & = & \U g(\Lambda) \U^*\\
	& = & \U  g(\Lambda)^{1/2} 
	\bigl( \U  g(\Lambda)^{1/2} \bigr)^*.
\end{eqnarray*}
Let $\b r_i$ be the $i$-th row of $\U  g(\Lambda)^{1/2}$, we immediately see that $\b r_i^T \b r_j = \bf G[i,j]$. More explicitly, these vectors are written in terms of the graph filter~:
\begin{equation*}
\b r_i[j] = \sum_\ell \sqrt{g(\b\lambda_\ell)} \b u_\ell[i] \b u_\ell[j]. 
\end{equation*}

These expressions suggest to define the Hilbert space $\rkhs$ as the closure of all linear combinations of localized graph filters $\T_i g$. This space is therefore composed of functions of the form~:
\begin{equation}
\label{eq:rkhs}
	\b x = \sum_{k \in \V} \alpha_k \T_kg.
\end{equation}
Note that any $\b x \in \rkhs$ has a well-defined graph Fourier transform~:

\begin{equation*}
\hat{\b x}(\ell) = g(\b \lambda_\ell) \sum_{k\in \V} \alpha_k \b u_\ell[k].
\end{equation*}

This allows to equip $\rkhs$ with following scalar product~:
\begin{equation*}
\langle \b x, \b y \rangle_{\rkhs}
= \sum_\ell \frac{1}{g(\b \lambda_\ell)} \hat{\b x}(\ell)^* \hat{\b y}(\ell)
\end{equation*}

and the vectors $r_i$ form an orthonormal basis of $\rkhs$:
\begin{eqnarray*}
	\langle \b r_i, \b r_j\rangle_{\rkhs} & = & 
	\sum_\ell \frac{1}{g(\b\lambda_\ell)} \sqrt{g(\b\lambda_\ell)} \b u_\ell[i]^* \sqrt{g(\b\lambda_\ell)} \b u_\ell[j] \\
	& = & \sum_\ell \b u_\ell[i]^* \b u_\ell[j] \\
	& = & \delta_{i,j}.
\end{eqnarray*}
Let us now see that $\rkhs$ is a reproducing kernel Hilbert space (rkhs). We  show that the scalar product with $\T_ig$ in $\rkhs$ is the evaluation functional at vertex $i$.  We first compute~:
\begin{eqnarray*}
	\langle \T_ig, \T_jg \rangle_{\rkhs} & = & \sum_\ell \frac{1}{g(\lambda_\ell)}  g(\lambda_\ell)^2 u_\ell[i]^* u_\ell[j] \\
	& = & \T_ig [j].
\end{eqnarray*}
By linearity of the scalar product and the definition of $\rkhs$ \eqref{eq:rkhs} we have~:
\begin{eqnarray*}
	\langle \T_ig, \b x \rangle_{\rkhs} & = & \sum_{k\in \V} \alpha_k \langle \T_ig, \T_kg\rangle_{\rkhs} \\
	& = & \sum_{k\in \V} \alpha_k \T_kg[i]\\
	& = & \b x[i].
\end{eqnarray*}
Finally, for any $\b x \in \rkhs$, $\b x = \sum_{k\in \V} \beta_k \T_kg$, we have the following explicit form of their norm~:
\begin{eqnarray*}
	\| \b x \|_{\rkhs}^2 & = & \langle \b x, \b x\rangle_{\rkhs} \\
	& = & \sum_\ell \frac{1}{g(\b\lambda_\ell)} g(\b\lambda_\ell)^2 \sum_{i,j \in \V} \beta_i \beta_j^* \b u_\ell[i] \b u_\ell[j]^* \\
	& = & \sum_{i,j \in \V} \beta_i \beta_j^*  \bigl( \sum_\ell g(\b\lambda_\ell) \b u_\ell[i] \b u_\ell[j]^* \bigr) \\
	& = & \sum_{i,j \in \V} \beta_i \bf G[i,j]\beta_j^* \\
	& = & \beta^T \bf G \beta .
\end{eqnarray*}

\paragraph{Transductive learning}
Now that we have established $\rkhs$ as a valid RKHS, we will seek to recover the full signal by solving the following problem~:

\begin{equation}
	\label{eq:transductive}
	\tilde{\b x} = \textrm{arg}\min_{\b x \in {\rkhs}} \sum_{k \in S} L(\b y_k,\b x[k]) + \mu \|\b x \|_{\rkhs}.
\end{equation}

Let us first decompose $\rkhs = \mathcal{H}_S \oplus \mathcal{H}_S^\perp$, where 

\begin{equation*}
\mathcal{H}_S = \left\{
\b x \in \rkhs \textrm{ s.t. }  \b x = \sum_{k \in S} \alpha_k \T_kg 
\right\}.
\end{equation*}

Let us note that, for any $\b x \in \mathcal{H}_S$, 
\begin{eqnarray*}
	\| \b x \|_{\rkhs}^2 & = & 
	\sum_\ell g(\b\lambda_\ell)
	\sum_{i,j \in S} \alpha_i \alpha_j^*
	\b u_\ell[i] \b u_\ell[j]^* \\
	& = &
	\sum_{i,j \in S} \alpha_i \alpha_j^*
	\sum_\ell g(\b\lambda_\ell) \b u_\ell[i] \b u_\ell[j]^* \\
	& = & \alpha^T \bf K \alpha
\end{eqnarray*}
where $\bf K[i,j] = \bf G[i,j]$, $i,j \in S$, is positive definite since it is a principal submatrix of a positive definite matrix.

Let $\b x \in \rkhs$ be decomposed as $\b x = \b x_S + \b x_{S^\perp}$, where $\b x_S$ (resp. $\b x_{S^\perp}$) is the orthogonal projection of $\b x$ on $\mathcal{H}_S$ (resp. $\mathcal{H}_S^\perp$). Now it is immediate to check that~:
\begin{eqnarray*}
	\langle \T_kg, \b x_{S^\perp} \rangle_{\rkhs} & =  & \b x_{S^\perp}[k] \\
	& = & 0, \,  \forall k \in S.
\end{eqnarray*} 
Inserting this relationship back into \eqref{eq:transductive}, we see that~:
\begin{equation*}
\sum_{k \in S} L(\b y_k, \b x_S[k]+ \b x_{S^\perp}[k]) 
+ \b\lambda \|\b x_S + \b x_{S^\perp}\|^2_{\rkhs}
\geq
\sum_{k \in S} L(\b y_k, \b x_S[k]) 
+ \lambda \| \b x_S\|^2_{\rkhs},
\end{equation*}
since $\b x_{S^\perp}[k]=0$ $\forall k \in S$ and adding $\b x_{S^\perp}$ can only increase the norm of $\b x_S$ in $\rkhs$. This shows that the minimizer of \eqref{eq:transductive} is in $\mathcal{H}_S$ and therefore of the form
\begin{equation*}
\tilde{\b x} = \sum_{k \in S} \beta_k \T_kg
\end{equation*}
for some coefficients  $\beta_k$. Moreover since $\|\tilde{\b x}\|_{\rkhs} = \beta^T \mathbf{K} \beta$, we can rewrite \eqref{eq:transductive} as a minimization only on those coefficients with $\tilde{\b x} = \mathbf{K} \tilde{\beta}$ and
\begin{equation}
	\label{eq:transductive_coeffs}
	\tilde{\beta} = \textrm{arg}\min_{\beta}
	\sum_k L(\b y_k, (\mathbf{K} \beta)[k]) +
	\mu \beta^T \mathbf{K} \beta.
\end{equation}
Finally, we observe that the recovered signal can be computed by filtering a stream of Kronecker deltas located at the observed values and weighted by the optimal coefficients computed in \eqref{eq:transductive_coeffs}~:
\begin{equation}
	\label{eq:regression}
	\tilde{\b x} = g( \L ) \left\{
	\sum_{k\in S} \tilde{\beta}_k  \delta_k
	\right\}.
\end{equation}
To summarize, in the case of the squared loss function $L(a,b) = (a-b)^2$, the transductive solution is given by the following two steps algorithm~:
\begin{enumerate}
	\item Compute the optimal coefficients $\tilde{\beta} = (\mathbf{K} + \lambda \mathbb{I})^{-1} \b y$
	\item Compute the regression $\tilde{\b x} = g(\L) \left\{
\sum_{k\in S} \tilde{\beta}_k  \delta_k
\right\}$.
\end{enumerate}
Note that in traditional ridge regression, the last step is usually given in terms of an explicit kernel that is easy to evaluate. In our case, this expression is also available from ~\eqref{eq:regression}:
\begin{eqnarray*}
\tilde{\b x}[i] &=& \sum_{k\in S}\tilde{\beta}_k \bf G[i,k]\\ 
&=& \sum_{k\in S}\tilde{\beta}_k \bf G[k,i] \\
&=& \sum_{k\in S}\tilde{\beta}_k \T_kg[i]
\end{eqnarray*}
and, while the kernel does not have a simple analytical form, the sum can be efficiently computed via a graph filtering algorithm. In particular, it is sufficient to perform $|S|$ filterings to get $\T_kg, \forall k \in S$. 

\subsection{Convex hull diffusion}

If we want to cast the general problem of transductive learning in a simpler framework, we can restrict ourselves to linear solutions of the form $\tilde{\b x} = \b A \b y$. This means finding the coefficients such as :

\begin{equation}\label{eq:linear_diffusion}
\tilde{\b x}[i] = \sum_{k\in S} \alpha_{i,k} \b y_k, 
\end{equation}

with $\alpha_{i,k} = \b A[i, k]$.

In the previous section, we just saw how a RKHS built on a graph filter $g$ allowed to weight the contributions of localized filters centered on a subset $S$ of vertices. Writing the answer as a linear solution such as defined in \eqref{eq:linear_diffusion} would give the following coefficients :
\begin{equation}
\alpha_{i,k} = \frac{\tilde{\beta}_k \T_kg[i]}{\b y_k}. 
\end{equation}
Of course, this is kind of a degenerate solution since the coefficients are normalized by $\b y_k$ and the optimal coefficients already contain the information from $\b y$. 

\subsubsection{Convex Hull Diffusion}

In this section we propose to use a notion of distances to the samples $\b y$ to set the coefficients, more formally $\alpha_{i,k} \simeq d(x_i, y_k)$ for some distance function $d$. Here, quite naturally, we propose to make use of the LKD as defined in Section~\ref{sec:metrics}. Since the coefficients $\alpha_{i,k}$ need to encode similarity between $i$ and $k$, a reasonable choice is to set :

\begin{equation}
\alpha_{i,k} = 1 - \lkd(i, k) = \frac{\T_ig^2[j]}{\|\T_ig\| \| \T_jg\|}. 
\end{equation}
 
Using this definition, we know that the coefficients $\alpha_{i,k}$ have good properties derived from Theorem~\ref{theo:lkd_pseudosemimetric}.  First, since the LKD has values in $[0, 1]$, the coefficients will also have values in this range. Second, $\alpha_{i,k} = \alpha_{k,i}$ which means that $\b A$ is symmetric, square and non-negative. Finally, for any kernel $g$ we have $\alpha_{i,i} = 1$ and, if we restrict ourselves to kernels as defined in Theorem~\ref{theo:lkd_semimetric}, we have $\alpha_{i,j} = 0 \Leftrightarrow i = j$. In general, we have the good property that the coefficients $\alpha_{i,k}$ will be small if the vertices $i$ and $k$ are far apart on the graph and big if they are close. 

Now, knowing that a classical problem related to embedding data in low dimension, and more specifically to data visualization is a concentration around zero, we wish to devise a method to prevent it. It is reasonable to suppose that the problem of concentration is often related to a lack of information about some points or an absence of normalization. For example, if we take the linear combination as defined in \eqref{eq:linear_diffusion}, this could happen if for some $i$, all the coefficients $\alpha_{i,k}$ are small. 

In order to avoid this problem, we propose to use a normalized version $\mathbf{\tilde{A}}$ of $\b A$ that maps the points $\b x$ in the convex hull of $\b y$. This is done simply by normalizing each line of $\b A$, that is :

\begin{equation}
\tilde{\alpha}_{i,k} = \frac{\alpha_{i,k}}{\sum_{k \in S} \alpha_{i,k}}
\end{equation}

with $\tilde{\alpha}_{i,k} = \mathbf{\tilde{A}}[i, k]$.

\section{Compressive Embedding}

Building on what has been presented in the previous sections, we now propose our main contribution, a compressive embedding algorithm. 


Algorithm~\ref{algo:ce} is the main algorithm of our proposed scheme. In the following, $D$ denotes the original $N \times K$ data matrix, $S$ the high-dimensional sketch, which is an $M \times K$ subset of $D$, $\mathcal{A}_e$ is any embedding algorithm, $E_S$ the low-dimensional sketch and $E_D$ an embedding of the full data $D$ being of dimension $M \times d$ and $N \times d$ respectively. $\mathcal{D}_{\G}$ is the diffusion operator on the graph. We have $M < N$, $d < K$ and typically $d = 2$ or $d = 3$ when targeting visualization tasks. 

\begin{algorithm}
	\caption{Compressive Embedding}
	\label{algo:ce}
	\begin{algorithmic}[1]
		\State Compute a knn graph $\G$ from the data $D$
		\State Sample $M$ nodes of $\G$ cf. Section~\ref{sec:adapted_sampling_scheme}
		\State Create a sketch $S$ from $D$ using the sampled nodes
		\State Apply $\mathcal{A}_e$ to $S$ to obtain an embedding $E_S = \mathcal{A}_e(S)$
		\State Solve the transductive learning problem to get $\mathcal{D}_{\G}$ c.f. Section~\ref{sec:transductive}
		\State Apply the diffusion operator to obtain the final embedding $E_D = \mathcal{D}_{\G}(E_S) $
	\end{algorithmic}
\end{algorithm}

Let us detail Algorithm~\ref{algo:ce} step by step. 
\begin{enumerate}
\item The graph construction can be carried out very efficiently by performing ANN searches in the data. Various methods and optimized libraries are available for this task such as FLANN  \cite{flann_pami_2014}\footnote{\url{http://www.cs.ubc.ca/research/flann/}} or ANNOY\footnote{\url{https://github.com/spotify/annoy}}. From our experiments, the graph construction process in not the main computationally intensive task. 
\item Guided by the theoretical analysis of Section~\ref{sec:adapted_sampling_scheme} we use low-pass concentrated kernels. Two choices are interesting, either a low-rank approximation (such as defined in Theorem~\ref{theo:sampling_tig_approx}) of a heat kernel $g(x) = e^{-\tau x}$ or an exponential window such as $g(x) = s\left(\frac{1-x}{b_{\text{max}} }\right)$ with :

\begin{equation*}
s(x)=\begin{cases} 0 & \mbox{if }x<-1 \\ \frac{e^{-\frac{a}{x}}}{e^{-\frac{a}{x}}+e^{-\frac{a}{1-x}}} & \mbox{if }x\in[-1,1]\\ 1 & \mbox{if }x>1 \end{cases}
\end{equation*}

where $b_{\text{max}}$ is the desired cut-off frequency. 

In Section~\ref{sec:adapted_sampling_scheme} we defined theoretically the number of samples needed to be able to sense and diffuse information from the sampled nodes to every other node. In practice, we were able to verify that $M = \Landau(\log(N))$, is sufficient for the diffusion process. When the number of classes $|\C|$ is available, $M = \Landau(|\C| \log(N))$ is a good choice. Otherwise $M = \Landau(d(\G)\log(N))$ is a valid alternative, with $d(\G)$ the diameter of the graph. All those choices for $M$ are above the bounds defined in Section~\ref{sec:adapted_sampling_scheme} for any choice of concentration of the kernels since $k < N$. 

\item Since there is a trivial mapping between node indices and data points, creating the high-dimensional sketch $S$ is simply taking the subset of $D$ corresponding to the samples indices. 

\item The compressive embedding framework does not impose any constraint on the type of algorithm used. Indeed, any embedding algorithm $\mathcal{A}_e$ that can be applied on $D$, can be applied on $S \subset D$. We note the application of the embedding algorithm $E_S = \mathcal{A}_e(S)$. 

\item The proposed transductive learning methods used for the diffusion need only graph filtering operations which are all carried out using Chebyshev polynomial approximations. The two operators that need to be computed are the localized filters $\T_ig$ and $\| \T_ig \|$. The former can be computed by filtering Kronecker delta centered on $i$, which means that exactly one filtering is needed to compute one $\T_ig$. The 2-norm $\| \T_ig \|$ being needed for all $i$, one cannot compute it trivially by computing $N$ atoms since it would require $N$ filterings. So instead of computing the exact solution, we can approximate it using random filtering, i.e. $\| \T_ig \|^2$ is well estimated by $\Esp{\|g(\L) \b R \delta_i \|^2}$ with $\b R$ an $N \times P$ random matrix. This estimator can be computed by performing only $P$ filterings. 

\item The final diffusion is a simple matrix-vector multiplication for both RKHS and CHD methods. 
\end{enumerate}


\section{Embedding quality measures}
\label{sec:embedquality}

In the context of embedding algorithms for visualization two approaches are often used to assess their quality. The first one is a purely qualitative assessment by visual examination, which generally implies to have access to labeled data (see e.g. \cite{maaten2008visualizing} \cite{tang2016visualizing}). When labels are not available, a common practice is to generate the labels using a clustering of the points in high dimension. Visual examination is especially used for relative quality assessment, i.e. one method versus others. 

A second method, which is not directly related to visualization, is to measure the quality of the embedding, i.e. if close high dimensional points stay close after embedding. Different numerical measures of local consistency have been proposed such as generalization error of 1-nearest neighbor classifiers \cite{van2009dimensionality}\cite{sanguinetti2008dimensionality}, trustworthiness and continuity \cite{venna2006visualizing}. These quantitative assessments do not take into account possible labels for the data. 

In order to have quantitative quality measures that take labels into account, we propose three methods that evaluate different characteristics of the embeddings. Note that, despite the face that we consider the problem settings for which the data points are associated to some categorical information, data points with no label or multiple labels can be easily accommodated. We will write the set of categorical labels (also called classes) as $\mathcal{C} = \{c_1, c_2, \ldots, c_k \}$. For each class $c_i$ we note $V_{c_i}$ the subset of vertices of $\G_e$ having the label $c_i$.  

The common point between all our proposed methods is that they are based on a similarity graph constructed between the points in the embedded domain, that we will call $\G_e$ to distinguish from $\G$. For simplicity, a simple kNN graph using the Euclidean distance on the embedded points is sufficient. The first method is inspired by Cheeger constants and measures the clusterability of $\G_e$. The second method uses diffusion distances to measure class homogenity and the third uses $\T_ig$ to estimate the amount of positional outliers. 

\subsection{Average Clusterability Index}

\paragraph{Graph cuts}

In order to use graph cuts, we start with a few definitions. A cut partitions a graph $\G$ in two complementary sets of vertices $S$ and $S^c$ with $V = S \cup S^c$ and $S \cap S^c = \emptyset$. The graph cut operator is then defined as 

\begin{equation}\label{eq:cut}
Cut(S, S^c) = \sum_{i \in S} \sum_{j \in S^c} \W_{ij}
\end{equation}

which represents the total weight of the edges between $S$ and $S^c$, or the weight of the edges trimmed by the cut. 

%
%

In order to define the balanced cuts we also need to use the volume operator which is defined as 

\begin{equation}
Vol(S) = \sum_{i \in S} d_i
\end{equation}

where $d_i$ is the degree of the vertex $v_i$. 

\paragraph{Balanced cuts}

The first interest of cuts in the context of clustering is that the minimization of \ref{eq:cut} happens to be a solution to the clustering problem \cite{wu1993optimal}. The minimal cut is however rarely used in practice as it tends to favor small sets of isolated vertices. This led to a shift in focus to balanced cuts, which are cuts normalized by the volume that balances the size of the clusters. Two of the most popular balanced cuts are the Cheeger cut \cite{cheeger1969lower} and the Normalized cut \cite{shi2000normalized}.  

The Cheeger cut is related to the Cheeger constant which is defined as :

\begin{equation}\label{eq:cheeger}
h(\G) = min_{S \nsubseteq V} \frac{Cut(S, S^c)}{\min(Vol(S), Vol(S^c))}
\end{equation}

for a graph $\G$. This number is a measure of the clusterability of $\G$, i.e. it is small if there is a strong bottleneck and large otherwise.  

\paragraph{Class clusterability}

The Cheeger cut and cheeger constant imply a minimization in order to find the best clusters, but in our case, we already have the clusters as they are derived from the labels. We can thus reformulate Eq.~\ref{eq:cheeger} to define a Cheeger score for a class $c_i$ as :

\begin{equation}
h(\G, c_i) = \frac{Cut(\V_{c_i}, \V_{c_i}^c)}{\min(Vol(\V_{c_i}), Vol(\V_{c_i}^c))}
\end{equation}

where $\V_{c_i} \subset \V$ is the subset of vertices whose label is $c_i$ and $\V_{c_i}^c \subset \V$ the complementary set containing all the other vertices. We note the number of vertices of a label $c_i$ as $N_{c_i} = | \V_{c_i} |$. Computing the above quantity for a given class give a measure of its clusterability from which we can define the Average Clusterability Index (ACI) as an average weighted by the classes cardinality :

\begin{equation}
\aci = \frac{1}{N} \sum_{c_i \in \C } N_{c_i} h(\G, c_i) =  \frac{1}{N} \sum_{c_i \in \C } N_{c_i} \frac{Cut(\V_{c_i}, \V_{c_i}^c)}{\min(Vol(\V_{c_i}), Vol(\V_{c_i}^c))} .
\end{equation}

This score, as it is inspired by the Cheeger constant, has similar properties : small values mean that the classes are well separated in the graph and large values mean that the classes are much more mixed. 

%
%

\subsection{Average Cluster Concentration}
\label{sec:acc}
The ACI introduced in the previous section serves to evaluate how clustrable are the different classes. However, this metric will not help discriminate between good clusterability with or without splitted classes. Take for example a dataset with ten classes (such as images of digits). Applying an embedding algorithm could result in having ten classes (the perfect case) or more, meaning that at least one class is splitted in more than one cluster. The ACI between the two cases should be almost indistinguishable, as both embedding scenarii will result in higly clusterable classes. 

In order to measure this effect, we need to measure the overall concentration of all points in a class, i.e. that all points in a class are reasonably close to each other. To this end, we introduce a new measure called Average Cluster Concentration which leverages the Kernelized Diffusion Distance introduced above. The principle is that the average distance of all pairs of points of a given class should be small if a class is well concentrated and larger if a class is splitted around different cluster centers. 

More formally, using the KDD as defined in \ref{eq:kdd_definition} and written $\kdd$, we define the ACC for one class $c_i \in \C$ as :

\begin{equation}
\acc(c_i) = \frac{1}{N_{c_i}^2} \sum_{v_i \in \V_{c_i}} \sum_{v_j \in \V_{c_i}} \kdd(v_i, v_j) .
\end{equation}

As was done above for the ACI, it is natural to give a final score by a weighted average over the classes :

\begin{equation}
\acc = \frac{1}{N} \sum_{c_i \in \C} N_{c_i}  \acc(c_i) = \frac{1}{N} \sum_{c_i \in \C} \frac{1}{N_{c_i}} \sum_{v_i \in \V_{c_i}} \sum_{v_j \in \V_{c_i}} \kdd(v_i, v_j).
\end{equation}

This direct computation of the $\acc$ is straightforward but requires $\Landau(N_{c_i}^2)$ distance evaluations per class. Using the original definition of the KDD, it means making at least $\Landau(N_{c_i})$ filterings, raising the complexity to $\Landau(N_{c_i} m |\E|) $ per class assuming order $m$ polynomial approximations for the filtering. Since this is too costly for large graphs, we propose to use a randomized version. 

%


An approach to accelerate the computation of the ACC is to estimate it by randomly picking pairs of points in the class. In order to be robust to different class sizes, we should take a number of samples proportional to $N_{c_i}$. If we assume that to evaluate $n_{c_i}$ pairs, a reasonable choice is to take $ n_{c_i} = \Landau(N_{c_i})$ which requires a linear number of distance evaluations instead of a quadratic number for the exact ACC computation.  

%
%



\section{Experiments} \label{sec:experiments}
In this section, we provide experiments whose objective is to show how our proposed methods behave in practice. The first experiments examine how the quantitative measures proposed in Section~\ref{sec:embedquality} perform on specially designed synthetic datasets. The second section of experiments allows to visualize the results of the compressive embedding routine using different diffusion operators and compared to state-of-the-art methods. 

The experiments were performed with the GSPBox~\cite{perraudin2014gspbox}, an open-source software. As we stand for reproducible research principles, our implementations and the code to reproduce all our results is open and freely available\footnote{Will be available online shortly. For now, please contact the corresponding author. }.
Since our methods use random signals, it is expected that the results shall be slightly different in the details, but overall consistent.


\subsection{Embedding quality measures}

In order to assess the validity of the quantitative measures proposed in Section~\ref{sec:embedquality} we use controlled synthetic datasets which exhibit the patterns we would like to measure. Since we want to evaluate embeddings the datasets are two-dimensional point clouds with labels. All are dynamic and can be deformed continuously between two conformations by varying a parameter $\lambda \in [0 ,1]$. Figure~\ref{fig:synthetic_data} displays all datasets for different values of $\lambda$. 

As can be seen, a unique design principle was used with different topological arrangements. The idea is that for $\lambda = 0$ the different classes are well separated in clusters, with a greater number of clusters than the number of classes. For $\lambda = 1$ the classes are well separated with each class corresponding exactly to one cluster. For intermediate values, the classes are mostly mixed as the points move between the $\lambda = 0$ and $\lambda = 1$ conformations. The checkerboard pattern has an intermediate non-mixed conformation at $\lambda = 0.5$.

Due to the randomness of the data generation process and the evaluation method of the $\acc$, all results are averages over multiple realisations. 

\begin{figure}
\includegraphics[width=\columnwidth]{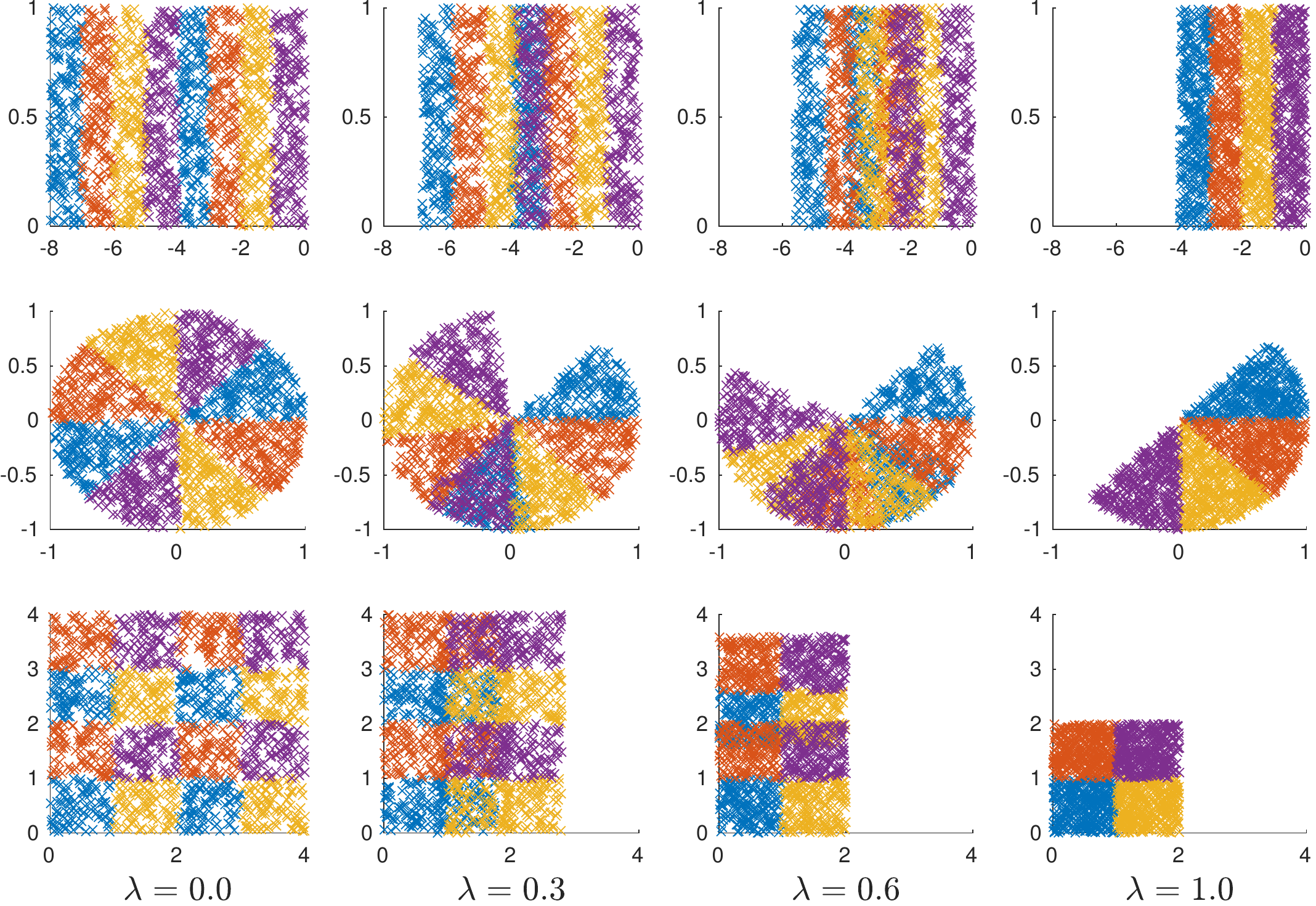}
\caption{Synthetic datasets with four classes, displayed at dynamics $\lambda = 0, 0.3, 0.6, 1$ (one value per column). On the top row, clusters form bands and move horizontally, on the medium row clusters form disc parts and rotate to form a half-disc and finally on the bottom row clusters are small squares in a larger one, move horizontally until $\lambda = 0.5$ and then vertically. }
\label{fig:synthetic_data}
\end{figure}

\subsubsection{ACI}

In this section, we expect to verify that the $\aci$ detects when classes are well clusterized. The results of the $\aci$ scores computed for the three synthetic datasets, using the full dynamic $\lambda \in [0, 1]$ and for different number of classes, is shown in Figure~\ref{fig:aci}. 

As expected, both extreme dynamics ($\lambda = 0$ and $\lambda = 1$ for bands and circle, and additionally $\lambda = 0.5$ for checkerboard) display low $\aci$ scores and the intermediate values correspond to the amount of mixing between the classes. In addition, more classes mean a steeper increase of $\aci$ the the classes mix. As a last remark, we can confirm that the $\aci$ is not sufficient to distinguish between splitted clusters and unified clusters ($\lambda = 0$ and $\lambda = 1$ respectively) which was the main reason for proposing the $\acc$.

\begin{figure}
\includegraphics[width=\columnwidth]{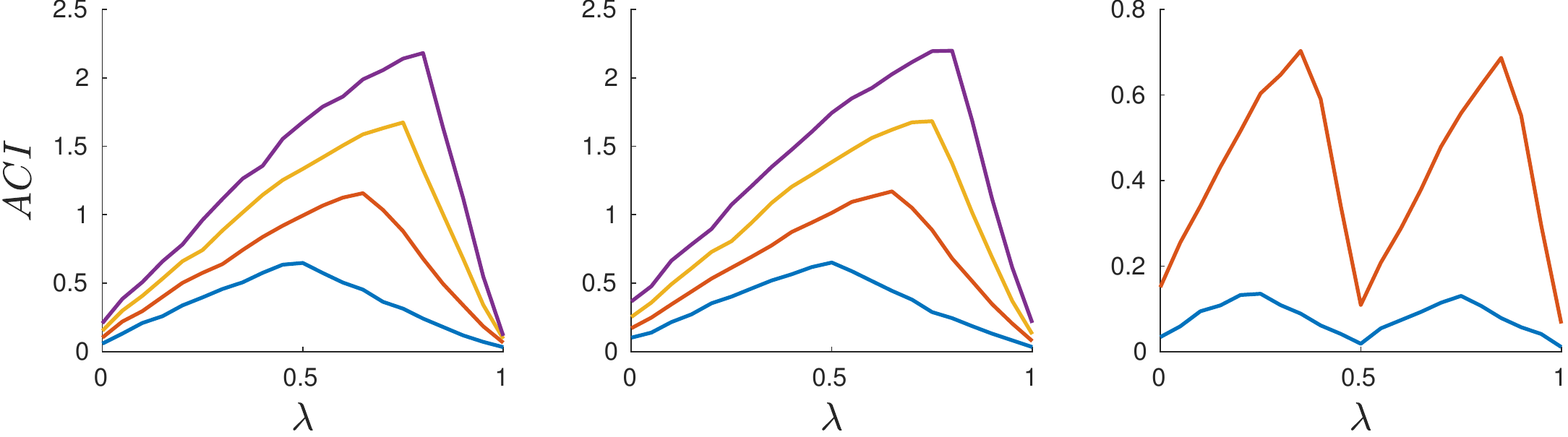}
\caption{ACI results on synthetic data for the bands (left), circle (middle) and checkerboard (right). The colors indicate the number of classes, blue $=2$, orange $=3$, yellow $=4$ and purple $=5$ for left and middle sub-figures, and blue $=4$ and orange $=16$ on the right. }
\label{fig:aci}
\end{figure}

\subsubsection{ACC}

In this experiment, we want to see if the $\acc$ is able to capture the notion of splitted clusters. Here, the $\acc$ was computed using the randomized method presented in Section~\ref{sec:acc}. The results for all datasets can be seen in Figure~\ref{fig:aci}. 

The first thing to note is that the curves are not perfectly smooth, due to the randomization process. The general behaviour is however quite clear, for every number of classes. Overall, the results are similar for all datasets and show that the $\acc$ allows to discriminate between $\lambda = 0$ for which we have higher values than for $\lambda = 1$. The result is particularly clear for the bands and checkerboard datasets, and less so for the circle.  
\begin{figure}
\label{fig:acc}
\includegraphics[width=\columnwidth]{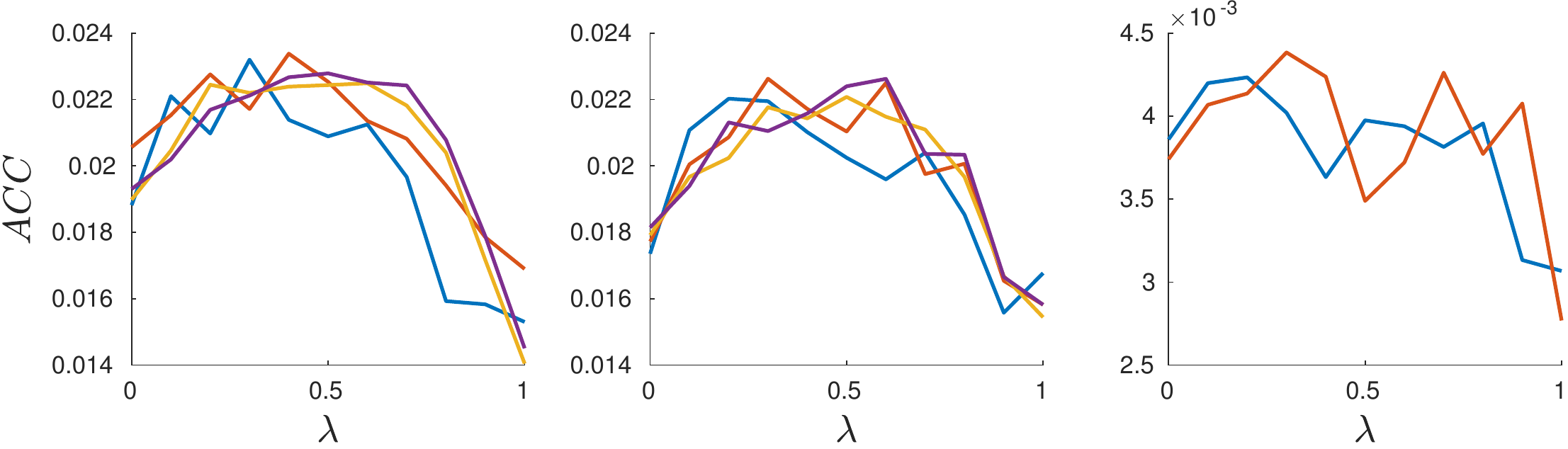}
\caption{ACC results on synthetic data for the bands (left), circle (middle) and checkerboard (right). The colors indicate the number of classes, blue $=2$, orange $=3$, yellow $=4$ and purple $=5$ for left and middle sub-figures, and blue $=4$ and orange $=16$ on the right.}
\end{figure}



\subsection{Real-world datasets visualization}

In this section, we will present two experiments on real-world datasets for visualization tasks. We restrict ourselves to a relatively small dataset $N < 10^5$ as some of the methods we evaluate cannot scale. We use the classical MNIST\footnote{\url{http://yann.lecun.com/exdb/mnist/}} dataset of handwritten digits. It contains 70'000 images of size $28\times28$. Note that for this size of dataset the sketch size was $550$, which means $0.008$\% of the data. 

\subsubsection{Visual comparison of diffusion operators}

In this first experiment, we show the resulting embedding of our proposed method using the different graph diffusion operators introduced in Section~\ref{sec:transductive}. As a baseline, we also included classical Tikhonov diffusion. Also, in addition to the CHD and RKHS methods, we show the result of bootstraping Tikhonov and RKHS diffusion with the result of the CHD. The visualizations provided by the 2D embeddings are shown in Figure~\ref{fig:realdata_variants}. 

Let us begin by inspecting the sketch. The different classes appears to be equally sampled and t-SNE provides a good embedding, while leaving a few overlapping clusters, one splitted class and a few outliers. The Tikhonov and RKHS diffusions achieve a radial separation of the classes but greatly suffer from concentration around zero. The CHD diffusion provides a good embedding similar to the sketch, but tends to produce too much overlaps. The use of bootstraping as displayed in the last two embeddings seem to improve the results of both Tikhonov an RKHS. By visual inspection, CHD appears to be the best diffusion operator, and in general the convex hull constraint seem to be working as expected. 

The quantitative scores for all methods are reported in Table~\ref{tab:acc_aci_variants}. The two worst ACI score are Tikhonov and RKHS, the best one is CHD and the bootstraped diffusion give medium values. This analysis corresponds well to visual inspection. The ACC scores are very similar and cannot discriminate well between the different methods. This is not surprising since there are no big class splits. 

The average timing for the entire process was 161s in total, from which  139s is spent in average on diffusion (step 5 and 6 of Algorithm~\ref{algo:ce}).

\begin{figure}
\includegraphics[width=\columnwidth]{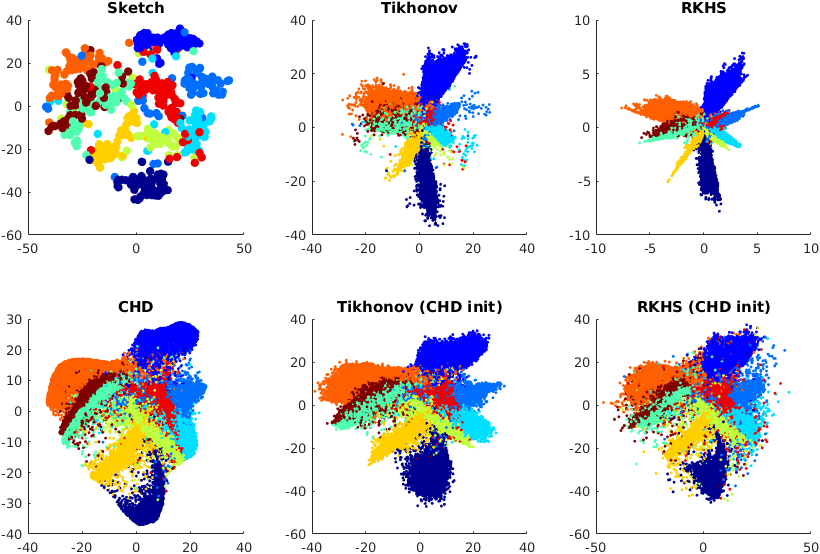}
\caption{MNIST visualisation using the Compressive Embedding method (with t-SNE as the inner embedding algorithm). The different colors corresponds to the ten different classes. }
\label{fig:realdata_variants}
\end{figure}

\begin{table}[t!]
\centering
\begin{tabular}{|c|c||c|c|c|c|c|}
\cline{2-7}
\multicolumn{1}{r|}{} & Sketch & Tikhonov & RKHS & CHD & Tik$+$CHD & RKHS$+$CHD \\ \cline{1-7}
ACI & 2.1035 & 3.2809 & 2.3214 & 1.1054 & 1.9223 & 1.6352 \\ \cline{1-7}
ACC & 0.0125 & 0.0312 & 0.0691 & 0.0490 & 0.0491 & 0.0393 \\ \cline{1-7}
\end{tabular}
\caption{ACI and ACC scores for different diffusion operators}
\label{tab:acc_aci_variants}
\end{table} 

\subsubsection{Original algorithms compared to Compressive Embedding}

In this last experiment, we want to see the behaviour of state-of-the-art and traditional visualization algorithm compared to Compressive Embedding versions. We report the visualizations produced, the computing time and the quantitative scores for four different algorithms : t-SNE\cite{maaten2008visualizing}, LargeVis\cite{tang2016visualizing}, Laplacian Eigenmaps\cite{belkin2003laplacian} and Sammon mapping\cite{sammon1969nonlinear}. 

The 2D embeddings produced are shown in Figure~\ref{fig:realdata_comparison}. If we first look at the original algorithms we can see that both t-SNE and LargeVis produce good embeddings as classes are well separated and clusters are strongly defined. A class split occurs for t-SNE and the repartition is not well balanced for LargeVis but the result is overall very good. Laplacian Eigenmaps gives a fair result but suffers from overlaps and concentration around zero. Sammon Mapping is not shown because the original implementation does not scale enough to complete on a dataset of this size. 

Now looking at the sketches we see that both t-SNE and LargeVis produce reasonably good embeddings while leaving a few overlaps, class splits and outliers. Laplacian Eigenmaps suffers from a bit of concentration around zero and tends to mix a few classes together. Sammon Mapping gives a result in which classes are fairly mixed and does not produce well defined clusters. 

Finally, the results of the CHD diffusion from the sketches is very consistant accross the different algorithms. Overall CE on t-SNE and LargeVis is quite satisfactory, giving well defined clusters. The downside being too much overlap and a lot of sparse outliers. While being satisfactory, the resulting embeddings are visually less good than their original counterparts. For Laplacian Eigenmaps the CE is very similar to the sketch and difficult to distinguish from its original counterpart. The CE of the Sammon Mapping is surprisingly good given the low quality of the sketch. Visually the result is better after diffusion, as the clusters are reasonably well defined. The problem of overlapping classes and sparse noise is still present. 

The ACI and ACC scores for all methods are reported in Table~\ref{tab:aci_realdata} and Table~\ref{tab:acc_realdata}. The lowest ACI are for original t-SNE and LargeVis, the second two best results are for CE t-SNE and CE LargeVis. Next, Laplacian Eigenmaps in its original implementation and with CE give similar ACI scores. Finally, Sammon Mapping gives the worst score. All values are very consistent with the visual inspection and tend to validate the use of the ACI as a quantitative measure for embedding quality evaluation. The values reported for the ACC are very similar and do not allow for a very good discrimination since no case of good clustering with class-split was present. 

Finally, the computing time is reported in Table~\ref{tab:timing_comparison}. For both t-SNE and Laplacian Eigenmaps, CE is one order of magnitude faster than the original implementations. In the case of LargeVis, the CE implementation is still faster but of a smaller factor. However, we need to evaluate this with caution as the original implementation of LargeVis is multi-threaded while all others implementations (including CE) is mono-thread. Taking into account the mono-thread computing time of LargeVis we go back to an order of magnitude acceleration.  

\begin{figure}
\includegraphics[width=\columnwidth]{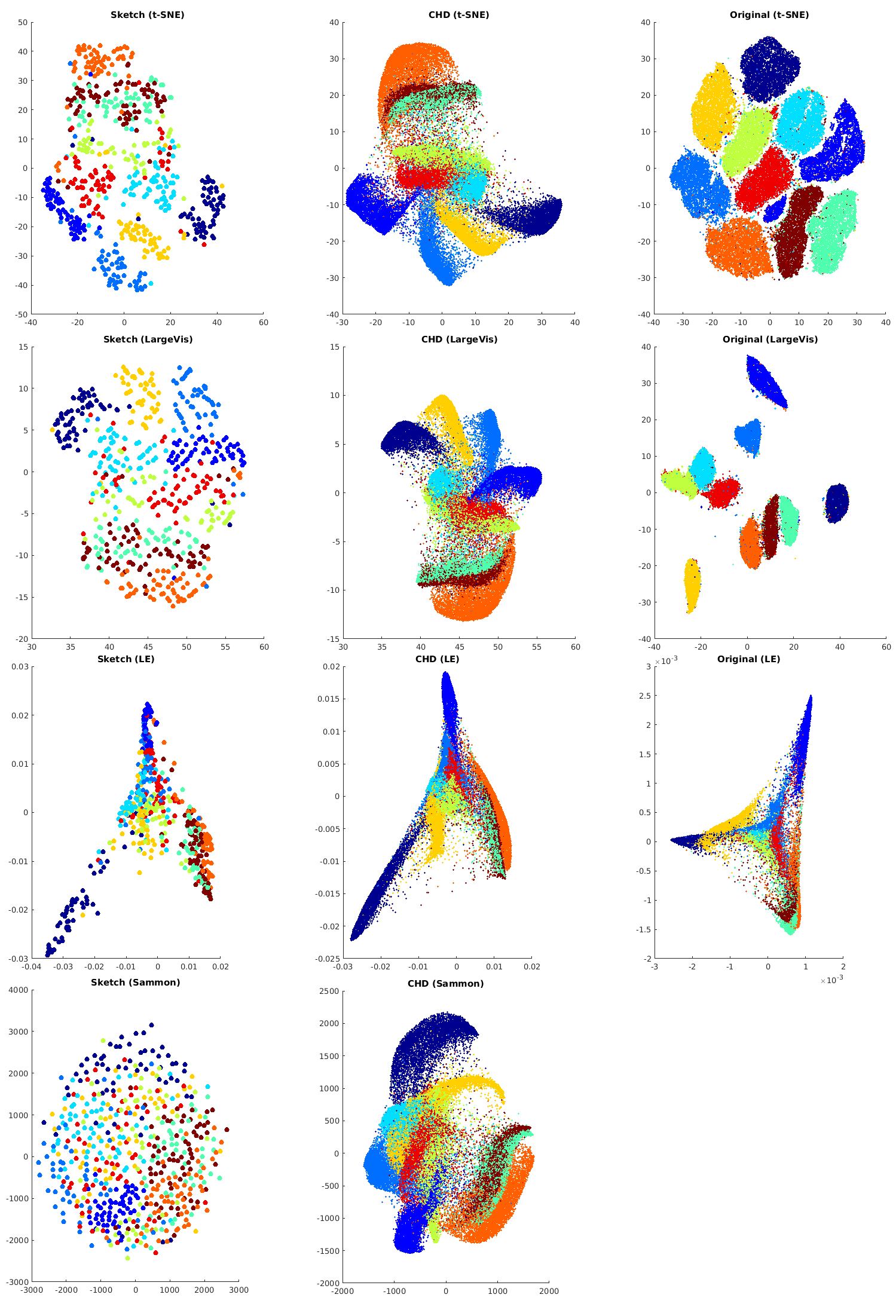}
\caption{MNIST visualization using different embedding algorithms both in their original implementations (right column) and using Compressive Embedding as an accelerator (middle column). The left column shows the result of the embedding algorithm on the sketch only. }
\label{fig:realdata_comparison}
\end{figure}

\begin{table}[t!]
\centering
\begin{threeparttable}
\begin{tabular}{|c|c|c|c|c|}
\cline{2-5}
\multicolumn{1}{r|}{ACI} & t-SNE & Laplacian Eigenmaps & Sammon Mapping & LargeVis \\ \cline{1-5}
Original & 0.30 & 2.88 & -\tnote{1} & 0.45 \\ \cline{1-5}
CE & 1.98 & 2.95 & 3.36 & 2.19 \\ \cline{1-5}
\end{tabular}
\caption{ACI scores comparison between original implementations and Compressive Embedding acceleration. }
\label{tab:aci_realdata}
\begin{tablenotes}
\item[1] exceeded the maximum memory available (128 GB)
\end{tablenotes}
\end{threeparttable}
\end{table}

\begin{table}[t!]
\centering
\begin{threeparttable}
\begin{tabular}{|c|c|c|c|c|}
\cline{2-5}
\multicolumn{1}{r|}{ACC} & t-SNE & Laplacian Eigenmaps & Sammon Mapping & LargeVis \\ \cline{1-5}
Original & 0.04 & 0.04 & -\tnote{1} & 0.03 \\ \cline{1-5}
CE & 0.05 & 0.05 & 0.04 & 0.04 \\ \cline{1-5}
\end{tabular}
\caption{ACC scores comparison between original implementations and Compressive Embedding acceleration. }
\label{tab:acc_realdata}
\begin{tablenotes}
\item[1] exceeded the maximum memory available (128 GB)
\end{tablenotes}
\end{threeparttable}
\end{table}

\begin{table}[t!]
\centering
\begin{threeparttable}
\begin{tabular}{|c|c|c|c|c|}
\cline{2-5}
\multicolumn{1}{r|}{Time [s]} & t-SNE & Laplacian Eigenmaps & Sammon Mapping & LargeVis \\ \cline{1-5}
Original & 1815 & 1666 & -\tnote{1} & 660\tnote{2} \\ \cline{1-5}
CE & 157 & 155 & 166 & 329 \\ \cline{1-5}
\end{tabular}
\caption{Computing time comparison between original implementations and Compressive Embedding acceleration. }
\label{tab:timing_comparison}
\begin{tablenotes}
\item[1] exceeded the maximum memory available (128 GB)
\item[2] the default implementation uses parallelism, the single thread time usage is 4090s. 
\end{tablenotes}
\end{threeparttable}
\end{table}


\section{Conclusion} \label{sec:conclusion}
In this contribution, we have presented a general framework for the acceleration of embedding and visualization algorithms. Our method is made possible by the use of similarity graphs, efficient sampling and graph diffusion. We showed how the method worked on real-world examples and that it gives satisfactory results while being one order of magnitude faster than original implementations. In future works we would like to evaluate active techniques both for sampling and for diffusion.

\appendix

\appendix

\section{Proofs}

\paragraph{Important lemmas.}
Let us first recall two important lemmas necessary for the proofs. The first one is a generalization of the Bernstein inequality for matrices.
\begin{lemma}[Matrix Bernstein: Bounded Case]\cite[Theorem 6.1]{tropp2012user} \label{lemma:bernstein_bounded}
Consider a finite sequence $\X_m$ of independent, random, self-adjoint
matrices with dimension $d$. Assume that
\begin{equation*}
\Esp{\X_m}=0 \quad \text{and} \quad \smax(\X_m) \leq R \quad \text{almost surely.}
\end{equation*}
Compute the norm of the total variance, 
\begin{equation*}
A^{2}:=\norm{\sum_m\Esp{\X_m^2}}_{op}
\end{equation*}
Then the following chain of inequalities holds for all $\delta\geq0$.
\begin{eqnarray*}
\Prob{\lmax\left( \sum_m\b{X}_m \right) \geq \delta} & \leq & d \cdot \exp\left(-\frac{A^2}{R^{2}}\cdot h\left(\frac{R\delta}{A^2}\right)\right)\\
 & \leq & d\cdot\exp\left(\frac{-\delta^{2}/2}{A^2+R\delta/3}\right)\\
 & \leq & \begin{cases}
d\cdot\exp(-3\delta^{2}/8 A^{2}) & \quad\text{for \ensuremath{\delta\leq A^2/R};}\\
d\cdot\exp(-3\delta/8R) & \quad\text{for \ensuremath{\delta\geq A^2/R}}.
\end{cases}
\end{eqnarray*}
where the function $h$ is defined as $h(u):=(1+u)\log(1+u)-u$ for $u\geq0$. 
\end{lemma}

The second lemma is a generalization of the triangular inequality for the norm of the localization operator.
\begin{lemma} \label{lemma:norm_tig_triangular_inequality}
Given any continuous kernel $g$ and $g^\prime$, the norm of the localization operator satisfies:
\begin{equation}
\norm{\T_{i}g^{\prime}}_{2}^{2}-\norm{\T_{i}\left(|g^{\prime}|-|g|\right)}_{2}^{2}
\quad \leq \quad \norm{\T_{i}g}_{2}^{2}\quad  \leq \quad 
 \norm{\T_{i}g^{\prime}}_{2}^{2}+\norm{\T_{i}\left(|g^{\prime}|-|g|\right)}_{2}^{2}
\end{equation}
\end{lemma}
\begin{proof}
From the definition of the localization operator, we have:
\begin{eqnarray}
\norm{\T_{i}g}_{2}^{2} & = & \sum_{\ell=0}^{N-1}g^{2}(\lel)\b{u}_{\ell}^{2}[i] \nonumber \\
 & = & \sum_{\ell=0}^{N-1}\left(g^{2}(\lel)-g^{\prime2}(\lel)\right)\b{u}_{\ell}^{2}[i]+\sum_{\ell=0}^{N-1}g^{\prime2}(\lel)\b{u}_{\ell}^{2}[i] \nonumber\\
 & \geq & \sum_{\ell=0}^{N-1}\left(g(\lel)-g^{\prime}(\lel)\right)^{2}\b{u}_{\ell}^{2}[i]+\sum_{\ell=0}^{N-1}g^{\prime2}(\lel)\b{u}_{\ell}^{2}[i] \label{eq:lemma_triangular_precision}\\
 & = & \norm{\T_{i}g^{\prime}}_{2}+\norm{\T_{i}\left(|g^{\prime}|-|g|\right)}_{2}. \nonumber
\end{eqnarray}
A simple change of variable concludes the proof. The inequality~\ref{eq:lemma_triangular_precision} follows
from the following assertion. For all $\lel$ such that $|g(\lel)|\leq|g^{\prime}(\lel)|$, we have
\begin{eqnarray*}
g^{2}(\lel) & = & g^{2}(\lel)-g^{\prime2}(\lel)+g^{\prime2}(\lel)\\
 & = & \left(|g(\lel)|-|g^{\prime}(\lel)|\right)\left(|g(\lel)|+|g^{\prime}(\lel)|\right)+g^{\prime2}(\lel)\\
 & \geq & -\left(|g^{\prime}(\lel)|-|g(\lel)|\right)\left(|g^{\prime}(\lel)|-|g(\lel)|\right)+g^{\prime2}(\lel)\\
 & = & g^{\prime2}(\lel)-\left(|g(\lel)|-|g^{\prime}(\lel)|\right)^{2}.
\end{eqnarray*}
For the $\lel$ such that $|g^{\prime}(\lel)|\leq|g(\lel)|,$ the inequality  $g^{2}(\lel) \geq g^{\prime2}(\lel)-\left(|g(\lel)|-|g^{\prime}(\lel)|\right)^{2}$ is trivially satisfied. 
\end{proof}

\paragraph{Proof of Theorem~\ref{theo:sampling_gL}}
The proof of Theorem~\ref{theo:sampling_gL} is inspired by~\cite[Theorem 2]{alaoui2014fast} but contains some subtleties.
\begin{proof}
Les us define $\balpha=\Uk^{*}\x$. We first we notice that 
\[
g(\L)\x=\Uk\Uk^{*}\x=\Uk g(\bLambda_{k})\balpha
\]
The quantity of interest is then rewritten as
\begin{eqnarray*}
&  & \frac{1}{M}\norm{\M\P^{-\frac{1}{2}}g(\L)\x}_{2}^{2}-\norm{g(\L)\x}_{2}^{2} \\
 & = & \frac{1}{M}\norm{\M\P^{-\frac{1}{2}}\Uk g(\b{\bLambda_{k}})\balpha}_{2}^{2}-\norm{\Uk g(\b{\bLambda_{k}})\balpha}_{2}^{2}\\
 & = & \balpha^{*}\left(\frac{1}{M}g(\bLambda_{k})\Uk^{*}\P^{-\frac{1}{2}}\M^{*}\M\P^{-\frac{1}{2}}\Uk g(\bLambda_{k})-g(\bLambda_{k})g(\bLambda_{k})\right)\balpha\\
 & = & \balpha^{*}\b{Y}\balpha
\end{eqnarray*}
 where $\b{Y}=\frac{1}{M}g(\bLambda_{k})\Uk^{*}\P^{-\frac{1}{2}}\M^{*}\M\P^{-\frac{1}{2}}\Uk g(\bLambda_{k})-g(\bLambda_{k})g(\bLambda_{k})$
. The remaining of the proof focus in characterizing the maximum and the
minimum eigenvalue of $\b{Y}.$ To do so, we decompose $\b{Y}$ into
a sum of $M$ independent, random, self-adjoint matrices$\X_{i}$
in order to apply Lemma~\ref{lemma:bernstein_bounded}. Let us define 
\[
\X_{i}:=\frac{1}{M}\left(g(\bLambda_{k})\Uk^{*}\left(\frac{\b{\delta}_{\omega_{i}}\b{\delta}_{\omega_{i}}^{*}}{\b{p}_{\omega_{i}}}-\b{I}\right)\Uk g(\bLambda_{k})\right).
\]
It can be verified that 
\begin{eqnarray*}
\b{Y} & =\sum_{i=1}^{M}\X_{i} & =\sum_{i=1}^{M}\left(\frac{1}{M}g(\bLambda_{k})\Uk^{*}\left(\frac{\b{\delta}_{\omega_{i}}\b{\delta}_{\omega_{i}}^{*}}{\b{p}_{\omega_{i}}}-\b{I}\right)\Uk g(\bLambda_{k})\right).
\end{eqnarray*}
By construction, the matrices $\X_{i}$ inherit independence from
the random variables $\b{\delta}_{\omega_{i}}$. Furthermore, we have
\begin{eqnarray*}
\mathbb{E}\left[\X_{i}\right] & = & \sum_{n=1}^{N}p_{n}\frac{1}{M}\left(g(\bLambda_{k})\Uk^{*}\left(\frac{\b{\delta}_{n}\b{\delta}_{n}^{*}}{\b{p}_{n}}-\b{I}\right)\Uk g(\bLambda_{k})\right)\\
 & = & \frac{1}{M}\left(g(\bLambda_{k})\Uk^{*}\left(\sum_{n=1}^{N}\b{\delta}_{n}\b{\delta}_{n}^{*}-\b{I}\right)\Uk g(\bLambda_{k})\right)\\
 & = & \b{0}=\mathbb{E}\left[-\X_{i}\right]
\end{eqnarray*}
To apply Lemma~\ref{lemma:bernstein_bounded} we need the maximum eigenvalue of $\X_{i}$ and
$-\X_{i}$.
\begin{eqnarray*}
\smax(\X_{i}) & = & \smax\left(\frac{1}{M}g(\bLambda_{k})\Uk^{*}\left(\frac{\b{\delta}_{\omega_{i}}\b{\delta}_{\omega_{i}}^{*}}{\b{p}_{\omega_{i}}}-\b{I}\right)\Uk g(\bLambda_{k})\right)\\
 & \leq & \frac{1}{M}\smax\left(\frac{1}{\b{p}_{\omega_{i}}}g(\bLambda_{k})\Uk^{*}\b{\delta}_{\omega_{i}}\b{\delta}_{\omega_{i}}^{*}\Uk g(\bLambda_{k})\right)\\
 & = & \frac{1}{M}\smax\left(\frac{1}{\b{p}_{\omega_{i}}}\b{\delta}_{\omega_{i}}^{*}\Uk g(\bLambda_{k})g(\bLambda_{k})\Uk^{*}\b{\delta}_{\omega_{i}}\right)\\
 & = & \frac{1}{M}\max_{i}\frac{1}{\b{p}_{i}}\b{\delta}_{i}^{*}\Uk g(\bLambda_{k})g(\bLambda_{k})\Uk^{*}\b{\delta}_{i}\\
 & = & \frac{1}{M}\max_{i}\frac{\norm{\T_{i}g}_{2}^{2}}{\b{p}_{_{i}}}
\end{eqnarray*}
\begin{eqnarray*}
\smax(-\X_{i}) & = & \smax\left(\frac{1}{M}g(\bLambda_{k})\Uk^{*}\left(\b{I}-\frac{\b{\delta}_{\omega_{i}}\b{\delta}_{\omega_{i}}^{*}}{\b{p}_{\omega_{i}}}\right)\Uk g(\bLambda_{k})\right)\\
 & \leq & \frac{1}{M}\smax\left(g^{2}(\bLambda)\right)=\frac{1}{M}\norm{g(\blambda)}_{\infty}^{2}
\end{eqnarray*}
Finally, before we can apply Lemma~\ref{lemma:bernstein_bounded}, we need to compute 
\begin{eqnarray*}
A^{2} & = & \smax\left(\mathbb{E}\left[\sum_{i=1}^{M}\X_{i}^{2}\right]\right)\\
 & = & \smax\left(\mathbb{E}\left[\frac{1}{M^{2}}\sum_{i=1}^{M}g(\bLambda_{k})\Uk^{*}\left(\b{I}-\frac{\b{\delta}_{\omega_{i}}\b{\delta}_{\omega_{i}}^{*}}{\b{p}_{\omega_{i}}}\right)\Uk g(\bLambda_{k})g(\bLambda_{k})\Uk^{*}\left(\b{I}-\frac{\b{\delta}_{\omega_{i}}\b{\delta}_{\omega_{i}}^{*}}{\b{p}_{\omega_{i}}}\right)\Uk g(\bLambda_{k})\right]\right)\\
 & = & \frac{1}{M}\smax\left(g(\bLambda_{k})\Uk^{*}\Esp{\left(\b{I}-\frac{\b{\delta}_{\omega_{i}}\b{\delta}_{\omega_{i}}^{*}}{\b{p}_{\omega_{i}}}\right)\Uk g(\bLambda_{k})g(\bLambda_{k})\Uk^{*}\left(\b{I}-\frac{\b{\delta}_{\omega_{i}}\b{\delta}_{\omega_{i}}^{*}}{\b{p}_{\omega_{i}}}\right)}\Uk g(\bLambda_{k})\right)\\
 & = & \frac{1}{M}\smax\left(g(\bLambda_{k})\Uk^{*}\left(\sum_{i=1}^{N}\frac{\norm{\T_{i}g}_{2}^{2}}{\b{p}_{_{i}}}\b{\delta}_{i}\b{\delta}_{i}^{*}\right)\Uk g(\bLambda_{k})\right)\\
 & \leq & \frac{1}{M}\norm{g(\lambda)}_{\infty}^{2}\max_{i}\frac{\norm{\T_{i}g}_{2}^{2}}{\b{p}_{_{i}}},
\end{eqnarray*}
since 
\begin{eqnarray*}
 & & \Esp{\left(\b{I}-\frac{\b{\delta}_{\omega_{i}}\b{\delta}_{\omega_{i}}^{*}}{\b{p}_{\omega_{i}}}\right)\Uk g(\bLambda_{k})g(\bLambda_{k})\Uk^{*}\left(\b{I}-\frac{\b{\delta}_{\omega_{i}}\b{\delta}_{\omega_{i}}^{*}}{\b{p}_{\omega_{i}}}\right)} \\
 & = & \sum_{i=1}^{N}\b{p}_{i}\left(\b{I}-\frac{\b{\delta}_{_{i}}\b{\delta}_{_{i}}^{*}}{\b{p}_{_{i}}}\right)\Uk g^{2}(\bLambda_{k})\Uk^{*}\left(\b{I}-\frac{\b{\delta}_{_{i}}\b{\delta}_{_{i}}^{*}}{\b{p}_{_{i}}}\right)\\
 & = & \sum_{i=1}^{N}\b{p}_{i}\b{\delta}_{_{i}}\b{\delta}_{_{i}}^{*}g^{2}(\L)\b{\delta}_{_{i}}\b{\delta}_{_{i}}^{*}-g^{2}(\L)\\
 & = & \sum_{i=1}^{N}\frac{\norm{\T_{i}g}_{2}^{2}}{\b{p}_{_{i}}}\b{\delta}_{_{i}}\b{\delta}_{_{i}}^{*}-g^{2}(L)\\
 & \preceq & \sum_{i=1}^{N}\frac{\norm{\T_{i}g}_{2}^{2}}{\b{p}_{_{i}}}\delta_{i}\delta_{_{i}}^{T}
\end{eqnarray*}
Let us denote $\max_{i}\frac{\norm{\T_{i}g}_{2}^{2}}{\b{p}_{_{i}}}=\alpha$.
We now apply Lemma~\ref{lemma:bernstein_bounded} to the $\b{Y}=\sum_{i=1}^{M}\X_{i}$ and we find
\[
\Prob{\frac{1}{M}\norm{\M\P^{-\frac{1}{2}}g(\L)\x}_{2}^{2}-\norm{g(\L)\x}_{2}^{2}\geq\delta\norm{\balpha}_{2}^{2}}\leq k\exp\left(-\frac{M\frac{\delta^{2}}{2}}{\alpha\left(\norm{g(\lambda)}_{\infty}^{2}+\frac{\delta}{3}\right)}\right).
\]
Similarly for $-\b{Y}=\sum_{i=1}^{M}-\X_{i},$ we find
\[
\Prob{\norm{g(\L)\x}_{2}^{2}-\frac{1}{M}\norm{\M\P^{-\frac{1}{2}}g(\L)\x}_{2}^{2}\geq\delta\norm{\balpha}_{2}^{2}}\leq k\exp\left(-\frac{M\frac{\delta^{2}}{2}}{\norm{g(\blambda)}_{\infty}^{2}\left(\alpha+\frac{\delta}{3}\right)}\right).
\]
In order to optimize the bound, we need to minimize $\alpha.$ Thus
we choose $p_{i}=\frac{\norm{\T_{i}g}_{2}^{2}}{\norm{g(\blambda)}_{2}^{2}}$
and we get $\alpha=\norm{g(\blambda)}_{2}^{2}.$ The two previous inequalities become
\[
\Prob{\frac{1}{M}\norm{\M\P^{-\frac{1}{2}}g(\L)\x}_{2}^{2}-\norm{g(\L)\x}_{2}^{2}\geq\delta\norm{\Uk^{*}\x}_{2}^{2}}\leq k\exp\left(-\frac{M\delta^{2}}{2\norm{g(\blambda)}_{2}^{2}\left(\norm{g(\lambda)}_{\infty}^{2}+\frac{\delta}{3}\right)}\right)
\]
\[
\Prob{\norm{g(\L)\x}_{2}^{2}-\frac{1}{M}\norm{\M\P^{-\frac{1}{2}}g(\L)\x}_{2}^{2}\geq\delta\norm{\Uk^{*}\x}_{2}^{2}}\leq k\exp\left(-\frac{M\delta^{2}}{2\norm{g(\blambda)}_{\infty}^{2}\left(\norm{g(\blambda)}_{2}^{2}+\frac{\delta}{3}\right)}\right)
\]
We make the change of variables $\delta^{\prime}\norm{g(\blambda)}_{\infty}^{2}=\delta$
\begin{equation*}
\Prob{\frac{\frac{1}{M}\norm{\M\P^{-\frac{1}{2}}g(\L)\x}_{2}^{2}-\norm{g(\L)\x}_{2}^{2}}{\norm{g(\blambda)}_{\infty}^{2}}\geq\delta^{\prime}\norm{\Uk^{*}\x}_{2}^{2}}\leq k\exp\left(-\frac{1}{2}\frac{\norm{g(\blambda)}_{\infty}^{2}}{\norm{g(\blambda)}_{2}^{2}}\frac{M\delta^{\prime2}}{\left(1+\frac{\delta^{\prime}}{3}\right)}\right) 
\end{equation*}
\begin{align}
\Prob{\frac{\norm{g(\L)\x}_{2}^{2}-\frac{1}{M}\norm{\M\P^{-\frac{1}{2}}g(\L)\x}_{2}^{2}}{\norm{g(\blambda)}_{\infty}^{2}}\geq\delta\norm{\Uk^{*}\x}_{2}^{2} } 
& \leq k\exp\left(-\frac{1}{2}\frac{M\delta^{\prime2}}{\left(\frac{\norm{g(\blambda)}_{2}^{2}}{\norm{g(\blambda)}_{\infty}^{2}}+\frac{\delta^{\prime}}{3}\right)}\right) \nonumber \\
& \leq k\exp\left(-\frac{1}{2}\frac{\norm{g(\blambda)}_{\infty}^{2}}{\norm{g(\blambda)}_{2}^{2}}\frac{M\delta^{\prime2}}{\left(1+\frac{\delta^{\prime}}{3}\right)}\right) \label{eq:one_side_tig_bound}
\end{align}
Finally, we substitute $\delta$ for $\delta^{\prime}$. We set the
success probability of the event
\[
\left|\frac{\frac{1}{m}\norm{\M\P^{-\frac{1}{2}}\U g(\bLambda)\x}_{2}^{2}-\norm{\U g(\bLambda)\x}_{2}^{2}}{\norm{g(\blambda)}_{\infty}^{2}}\right|\geq\delta\|\x\|_{2}^{2}.
\]
 to $1-\epsilon$.
As both sides of the bound have to be taken into account, we need
\[
\frac{\epsilon}{2}\ge k\exp\left(-\frac{1}{2}\frac{\norm{g(\lambda)}_{\infty}^{2}}{\norm{g(\blambda)}_{2}^{2}}\frac{Mt^{2}}{\left(1+\frac{\delta}{3}\right)}\right),
\]
which is equivalent to impose on $M$
\[
M\geq2\frac{1}{\delta^{2}}\frac{\norm{g(\blambda)}_{2}^{2}}{\norm{g(\blambda)}_{\infty}^{2}}\left(1+\frac{\delta}{3}\right)\log\left(\frac{2k}{\epsilon}\right)
\]
\end{proof}

\paragraph{Proof of Theorem~\ref{theo:sampling_tig_exact}}
\begin{proof}
Given $M\geq2\frac{1}{\delta^{2}}\frac{\norm{g(\blambda)}_{2}^{2}}{\norm{g(\blambda)}_{\infty}^{2}}\left(1+\frac{\delta}{3}\right)\log\left(\frac{k}{\epsilon}\right)$, we use~\eqref{eq:one_side_tig_bound} and set $\x = \b{\delta}_i$. Then with a probability $\epsilon$, we have
\[
\frac{\norm{\T_{i}g}_{2}^{2}}{\norm{g(\blambda)}_{\infty}^{2}}
 - \frac{\frac{1}{M}\norm{\M\P^{-\frac{1}{2}}\T_{i}g}_{2}^{2}}{\norm{g(\blambda)}_{\infty}^{2}}
 \geq \delta\norm{\Uk^{*}\b{\delta}_{i}}_{2}^{2}.
\]
As a result, with a probability $1-\epsilon$, we have
\[
\frac{\frac{1}{M}\norm{\M\P^{-\frac{1}{2}}\T_{i}g}_{2}^{2}}{\norm{\T_{i}g}_{2}^{2}}\geq1-\delta \frac{\norm{g(\blambda)}_{\infty}^{2}\norm{\Uk^{*}\b{\delta}_{i}}_{2}^{2}}{\norm{\T_{i}g}_{2}^{2}}.
\]
The change of variable $\delta^{\prime}=\delta\frac{\norm{g(\blambda)}_{\infty}^{2}\norm{\Uk^{*}\b{\delta}_{i}}_{2}^{2}}{\norm{\T_{i}g}_{2}^{2}}$ concludes the proof.
For the factor $\frac{\delta}{3}$, we use the fact that $\frac{\norm{g(\blambda)}_{\infty}^{2}\norm{\Uk^{*}\b{\delta}_{i}}_{2}^{2}}{\norm{\T_{i}g}_{2}^{2}}\geq$1. 
\end{proof}

\paragraph{Proof of Theorem~\ref{theo:sampling_tig_approx}}
\begin{proof}
We first use the fact that $\norm{\b{A}\T_{i}g^{\prime}}_{2} \geq \norm{\b{A}\T_{i}g}_{2}$ for any linear operator $\b{A}$. This comes from the fact that $\T_i g $ for a fixed $i$ can be written as $\b{T}_i g(\blambda)$ where $\b{T}_i$ is a linear operator. We successively apply Theorem~\ref{theo:sampling_tig_exact} and in similar way to Theorem~\ref{theo:sampling_tig_exact}, Equation~\ref{eq:one_side_tig_bound} to obtain
\begin{eqnarray*}
\frac{1}{M}\norm{\M\P^{\frac{1}{2}}\T_{i}g}_{2}^{2} & \geq & \frac{1}{M}\norm{\M\P^{\frac{1}{2}}\T_{i}g^{\prime}}_{2}^{2}\\
 & \geq & \norm{\T_{i}g^{\prime}}_{2}^{2}-\delta\norm{g^{\prime}(\blambda)}_{\infty}^{2}\norm{\Uk^{*}\b{\delta}_{i}}_{2}^{2}\\
 & \geq & \norm{\T_{i}g}_{2}^{2}-\norm{\T_{i}\left(|g^{\prime}|-|g|\right)}_{2}^{2}-\delta\norm{g^\prime(\blambda)}_{\infty}^{2}\norm{\Uk^{*}\b{\delta}_{i}}_{2}^{2},
\end{eqnarray*} for a number of samples 
\begin{equation*}
M\geq2\frac{1}{\delta^{2}}\frac{\norm{g^\prime(\blambda)}_{2}^{2}\norm{g^\prime(\blambda)}_{\infty}^{2}\norm{\Uk^{*}\b{\delta}_{i}}_{2}^{4}}{\norm{\T_{i}g^\prime}_{2}^{4}}\left(1+\frac{\delta}{3}\right)\log\left(\frac{k}{\epsilon}\right).
\end{equation*}
The change of variable $\delta^{\prime}=\delta\frac{\norm{g^\prime(\blambda)}_{\infty}^{2}\norm{\Uk^{*}\b{\delta}_{i}}_{2}^{2}}{\norm{\T_{i}g}_{2}^{2}}$
and the division by $\norm{\T_{i}g}_{2}^{2}$ conclude the proof.
For the factor $\frac{\delta}{3}$, we use the fact that $\frac{\norm{g^\prime(\blambda)}_{\infty}^{2}\norm{\Uk^{*}\b{\delta}_{i}}_{2}^{2}}{\norm{\T_{i}g}_{2}^{2}}\geq$1.
\end{proof}

\section*{Acknowledgment}
We would like to thank Lionel Martin for valuable discussions. 

\bibliographystyle{ieeetr}
\bibliography{biblio}

\end{document}